\documentclass{article} % Anonymized submission
\usepackage[final]{mynips_2016}

\usepackage[utf8]{inputenc} % allow utf-8 input
\usepackage[T1]{fontenc}    % use 8-bit T1 fonts
\usepackage{hyperref}       % hyperlinks
\usepackage{url}            % simple URL typesetting
\usepackage{booktabs}       % professional-quality tables
\usepackage{amsfonts}       % blackboard math symbols
\usepackage{nicefrac}       % compact symbols for 1/2, etc.
\usepackage{microtype}      % microtypography

\newcommand{\withappendix}[2]{#1} %uncomment to include appendix and ommit  references to supplement

\title{Learning Combinations of Sigmoids\\ Through Gradient Estimation}

\author{
  Stratis Ioannidis\\
  Dept.~of Electrical and Computer Engineering\\
  Northeastern University\\
  Boston, MA, 02115\\
  \texttt{ioannidis@ece.neu.edu} \\
   \And
  Andrea Montanari\\
  Dept.~of Electrical Engineering and Dept.~of Statistics\\
 Stanford University\\
 Stanford, CA 94304\\
  \texttt{montanari@stanford.edu}\\
}

\usepackage{algorithm}
\usepackage{algorithmic}
\usepackage{hyperref,natbib,graphicx,color}

\usepackage{amsmath,amssymb,amsthm} 
\usepackage{verbatim}

\newtheorem{claim}{Claim}[section]
\newtheorem{lemma}[claim]{Lemma}

\newtheorem{theorem}{Theorem}

\newtheorem{corollary}[claim]{Corollary}

\newcommand{\si}[1]{\textcolor{red}{\bf [Stratis:  #1]}}
\newcommand{\upperb}[1]{\overline{#1}}
\newcommand{\lowerb}[1]{\underline{#1}}

\def\cR{{\mathcal{R}}}
\def\cB{{\mathcal{B}}}
\def\cG{{\mathcal{G}}}

\def\de{{\rm d}}
\def\reals{{\mathbb R}}

\def\bu{{\boldsymbol{u}}}

\def\bM{{\boldsymbol{M}}}
\def\bv{{\boldsymbol{v}}}

\def\bow{{\boldsymbol{\overline{w}}}}
\def\uf{{\upperb{f}}}
\def\lf{{\lowerb{f}}}
\def\cC{{\mathcal C}}

\def\bx{{\boldsymbol{x}}}
\def\bX{{\boldsymbol{X}}}
\def\bhx{{\boldsymbol{\hat{x}}}}
\def\bhxi{{\boldsymbol{\hat{\xi}}}}
\def\bhw{{\boldsymbol{\hat{w}}}}
\def\btw{{\boldsymbol{\tilde{w}}}}

\def\bw{{\boldsymbol{w}}}
\def\bd{{\boldsymbol{d}}}

\def\bxi{{\boldsymbol{\xi}}}
\def\bXi{{\boldsymbol{\Xi}}}
\def\bzeta{{\boldsymbol{\zeta}}}
\def\ba{{\boldsymbol{a}}}
\def\bb{{\boldsymbol{b}}}

\def\<{\langle}
\def\>{\rangle}
\def\prob{{\mathbb P}}

\def\E{{\mathbb E}} %expectation
\def\Var{{\sf{Var}}}
\def\Cov{{\sf{Cov}}}
\def\eps{{\varepsilon}}

\def\normal{{\sf N}}
\def\id{{\rm I}}
\def\Kexp{K_{{\rm exp}}}
\def\ind{{\mathbb{I}}}
\def\arc{{\mathop{\rm arc}}}

\DeclareMathOperator{\linspan}{span}

\newcounter{packednmbr}
\newenvironment{packedenumerate}{\begin{list}{\thepackednmbr.}{\usecounter{packednmbr}\setlength{\itemsep}{0pt}\addtolength{\labelwidth}{-5pt}\setlength{\leftmargin}{\labelwidth}\setlength{\listparindent}{\parindent}\setlength{\parsep}{0pt}\setlength{\topsep}{3pt}}}{\end{list}}
\newenvironment{packeditemize}{\begin{list}{$\bullet$}{\setlength{\itemsep}{0pt}\addtolength{\labelwidth}{-5pt}\setlength{\leftmargin}{\labelwidth}\setlength{\listparindent}{\parindent}\setlength{\parsep}{0pt}\setlength{\topsep}{3pt}}}{\end{list}}

\begin{document}

\maketitle

\begin{abstract}
We develop a new approach to learn the parameters of regression models
with hidden variables. In a nutshell, we estimate the gradient of the
regression function at a set of random points, and cluster the
estimated gradients. The centers of the clusters are used as estimates
for the parameters of hidden units.
We justify this approach by studying a toy model, whereby the
regression function is a linear combination of sigmoids. We prove that
indeed the estimated  gradients concentrate around the parameter
vectors of the hidden units, and provide non-asymptotic bounds on the
number of required samples. To the best of our knowledge, no comparable
guarantees have been proven for linear combinations of sigmoids.
\end{abstract}

%\begin{keywords}
%hidden variables, neural networks, parameter estimation
%\end{keywords}

\section{Introduction}

Classification and regression models with hidden
variables have a long history in statistical learning. They naturally arise when learning mixtures, a topic recently receiving considerable attention~\citep{chaganty2013spectral,sun2014learning,anandkumar2012learning,anandkumar2014tensor,hsu2013learning}.
Interest on such models has also increased because of the empirical
success of deep neural networks in image and speech processing tasks 
\citep{bengio2009learning,krizhevsky2012imagenet,hinton2012deep}.
One of the most striking properties of these models is the ability to
learn high-level representations that are particularly
useful for discriminative purposes \withappendix{\citep{boureau2008sparse,mairal2009supervised,boureau2010learning,yu2013feature,humphrey2013feature}.}{\citep{mairal2009supervised,boureau2010learning,yu2013feature}.}
This ability--often referred to as `feature learning'--is yet poorly
understood. From a modeling point of view, it is unclear what are the key
elements of such high-level representations, and how are they captured (for
instance) by deep neural networks. From a computational point of view,
the corresponding empirical risk minimization problem is highly
non-convex, and it is unclear why existing algorithms are empirically
successful at learning these representations. 

In this work, we consider a regression model with response variable 
$y\in \reals$ and covariates $\bx\in\reals^d$, linked through 
the regression function:
\begin{align}
\E\big\{y|\bx\big\} \equiv r(\bx) =\textstyle\sum_{\ell=1}^k u_\ell\, f(\<\bw^\ell,\bx\>)\,,\label{eq:Model}
\end{align}
where, $\bw^1,\dots,\bw^k\in\reals^d$, $u_1,\dots,u_k\in\reals$, 
 $f:\reals\to\reals$ is the sigmoid 
$f(x) = \tanh(\beta x) = \frac{e^{\beta x} -e^{-\beta x}}{e^{\beta x} + e^{-\beta x}},$ for some  $\beta>0$, and
 $\<\ba,\bb\>=\sum_{i=1}^da_ib_i$ is the usual scalar product
in $\reals^d$.  In the \emph{general case},  the weights $u_\ell$, $\ell\in [k]\equiv\{1,\ldots,k\}$ have arbitrary signs, while in the \emph{mixture case}, they are positive and sum to one.
%We further assume  $\bX\sim \normal(0,\id_{d\times d})$.
%Given $n$ independent pairs $(\bx^{(1)},y^{(1)})$, \dots, $(\bx^{(n)},y^{(n)})$ generated from this distribution,
Our objective is  to learn the parameter vectors
$\{\bw^{\ell}\}_{\ell=1}^k$, particularly when 
 $k\ll d$; it is useful to pause for a  few remarks on this model:
\begin{packeditemize}
\item In the \emph{general case}, learning the parameter vectors $\{\bw^\ell\}_{\ell=1}^{k}$  can be viewed as a simple model for feature learning. In particular, \eqref{eq:Model}  corresponds to a two-layer neural network with
$k$ hidden units. The non-linear functions $f(\<\bw^1,\cdot\>)$, $
f(\<\bw^2,\cdot\>)$, \dots, $f(\<\bw^k,\cdot\>)$, provide a
high-level, lower-dimensional representation of the data. %Once this representation is known, learning
%the regression function is straightforward. 
%
%\item The feature vector $\bx$ contains no information  about the regression function, by itself. This is in contrast with part of the literature on neural networks \citep{lee2009unsupervised,boureau2010learning}, whereby unsupervised learning is found to be useful for constructing high-level representations.  Understanding the role of unsupervised learning for extracting discriminative features is an outstanding challenge. Nevertheless, we focus here on the supervised component.
\item In the \emph{mixture case}, \eqref{eq:Model} is the expected label generated by a mixture of $k$ logistic classifiers, each selected with probability $u_\ell$.  Each vector $\bw^{\ell}$ is the normal to the separating hyperplane defining each classifier; learning  $\{\bw^{\ell}\}_{\ell=1}^k$ thus corresponds to learning the mixture's constituent distributions. When $k\ll d$,  the number of classifiers (or `modes') is smaller than the ambient dimension, as is the case in many applications~\citep{sun2014learning,icml2014c2_yia14,chen2014convex}.
\item In both cases, once $\{\bw^{\ell}\}_{\ell=1}^{k}$ are known, learning the full
 regression function is straightforward: fitting $\{u_\ell\}_{\ell=1}^{k}$
is a standard linear regression problem. 
%
%\item The Gaussian model $\bx\sim \normal(0,\id_{d\times d})$ for the feature vectors is admittedly idealized. We believe that our approach can be generalized to a broader setting (namely sub Gaussian features), at the price of more involved proofs. We defer this generalizations to future work, briefly commenting on the required changes in the algorithm in Section \ref{sec:Method}.
%
\end{packeditemize}
Our approach is based on a simple remark.
The gradient of the regression function $r$  at any $\bxi\in\reals^d$ is
a linear combination of $\{\bw^{\ell}\}_{\ell =1}^k$, with coefficients
depending on $\bxi$, i.e.,
%
%\begin{align}
$\nabla r(\bx) = \sum_{\ell=1}^{k} c_{\ell}(\bxi) \, \bw^{\ell}\, $
%\end{align}
%
(see Section~\ref{sec:intuition}).
Further, if $\bxi$ is  sufficiently far from the origin, this linear
combination is typically \emph{sparse}: it contains at most one  non-vanishing coefficient. Our algorithm thus proceeds in two steps:
$(1)$ estimate the gradient $\nabla r(\,\cdot\,)$ at $m_0$ random positions
$\bxi^1$, \dots, $\bxi^{m_0}$; $(2)$ cluster these estimates and use 
cluster centers as estimates of $\{\bw^\ell\}_{\ell=1}^k$.

Our main technical contribution is to prove that this approach is
consistent: for large  $m_0$, the algorithm
generates gradient estimates that  concentrate around $\bw^1$, \dots, $\bw^{k}$. We establish non-asymptotic bounds on the
minimum sample size that guarantees this clustering to take place. We do so under three different methods for estimating $\nabla r(\cdot)$.
Assuming access to a value oracle for $r$,  we construct a gradient estimator under which clustering succeeds with only $\Theta(d)$ oracle calls. Assuming access to covariates $\bx$ sampled from a standard Gaussian distribution, we show that clustering succeeds with access to $e^{\Theta(d)}$ samples in the general case. A dimensionality reduction method by \cite{sun2014learning} allows us to reduce the complexity to $\Theta(d)+e^{\Theta(\mathtt{poly}(k))}$ samples, in the mixture case (i.e., $\{u_\ell\}_{\ell=1}^{k}$ positive and summing to one).
%\si{REWRITE THIS TO HIGHLIGHT NEW RESULTS AND DIFFERENT WAYS OF ESTIMATING GRADIENT. %We anticipate that the sample complexity bounds produced by our
%analysis are not fully satisfactory. Indeed, we prove that for our
%algorithm to succeed it is sufficient that $n\ge
%e^{\Theta(d)}$. 
%We believe that a better choice of some of
%the components in the algorithm will lead to significantly smaller
%complexity, as discussed in Section \ref{sec:Mainresult}. Also the
%computational complexity is only polynomial in the dimensions, and
%number of hidden units. Finally, we are not aware of comparable
%consistency results in the literature, even for $n=\infty$ (see next section). 

The rest of the paper is organized as follows. In
Section~\ref{sec:relatedwork}, we review related work in the area. 
Section~\ref{sec:Method} describes in detail the new algorithm. 
In Section \ref{sec:Mainresult} we state our main results.  Finally,
we outline our proof in Section \ref{sec:Proof}, with many technical
details deferred to the~\withappendix{appendix.}{supplement.}

\section{Related Work} \label{sec:relatedwork}

Typical approaches to learning mixtures, like
EM~\citep{dempster1977maximum} come with no guarantees, and
suffer from convergence to local minima. Providing
guarantees for even the
idealized case of learning mixtures of Gaussians is non-trivial, and has been the subject of several recent
studies~\citep{moitra2010settling,hsu2013learning}.  
There are relatively few rigorous results that guarantee learning 
for regression models with latent variables. 
 \cite{chaganty2013spectral} consider mixtures of linear
regressions. In this setting, they show that regressing the
response from second and third order tensors of the
covariates yields coefficients, also higher order tensors,
whose decomposition reveals the model parameters. A different approach, relying only on the second-order tensor
(i.e., the covariance) and alternating minimization is followed by
\cite{icml2014c2_yia14} ,
for a mixture composed of two linear models in the absence of
noise; the same setting, in the presence of noise, is
studied by \cite{chen2014convex}. 
None of these approaches can be applied to our model:
our components are non-linear (sigmoids), while the above works %of \cite{chaganty2013spectral,icml2014c2_yia14,chen2014convex}
 focus on
linear components; %$(ii)$ In these works,
%\cite{chaganty2013spectral,icml2014c2_yia14,chen2014convex},
%the second layer chooses a random hidden unit, while in our model it computes
%a linear combination; $(iii)$ 
moreover, both \cite{icml2014c2_yia14} and
\cite{chen2014convex} limit their analysis to $k=2$ hidden units.

\cite{sedghi2014provable}  apply the moment method, as well as
tensor factorization techniques, to learn  mixtures of  sigmoids. 
Their contribution  is not directly comparable to ours:
%$(i)$~\cite{sedghi2014provable}  %consider a mixture model whereby 
%the second layer chooses a random hidden unit; $(ii)$ They
they assume the
vectors $\bw^{\ell}$ are random, and require a special
non-degeneracy condition on the expectations of third derivatives of
the hidden units. For instance, this condition is not satisfied by
$f(x) = \tanh(\beta x)$ if $\bx$ has a symmetric distribution. 
 \cite{sun2014learning} consider the
problem of learning a mixture of linear
classifiers, and provide guarantees for learning the subspace
spanned by $\{\bw^\ell\}_{\ell = 1}^k$. This, in turn, can
be used for dimensionality reduction: projecting the
covariates to this space reduces the problem dimension from $d$ to $k$. We focus on
learning the parameter vectors, rather than their span; our contribution  is thus complementary to \cite{sun2014learning}. In fact, we exploit their result to reduce our algorithm's complexity under Gaussian covariates.
%Our algorithm can be executed over 
%covariates after they have been projected to the $k$-dimensional space, as learned via the algorithm of 
%\cite{sun2014learning}. 
 
Our approach can be cast as a means to learn the parameters
of a two-stage neural network. Such networks are known to be
quite expressive~\citep{barron1993universal}, in that they can
approximate arbitrary polynomials. A recent result by
\cite{andoni2014learning} shows that, in the presence of a large
number of neurons with random parameter vectors, a polynomial can be learned through
gradient descent. Our contribution differs in two important
directions: $(i)$ we want  to learn the hidden-unit parameter vectors, rather than
approximate a given regression function, and $(ii)$ we develop explicit bounds on the
sample size, while \cite{andoni2014learning} assume infinite sample
size $n=\infty$. Namely, they assume  access to the gradient of
the regression function: a substantial part of our technical work
is devoted to proving  this can be sufficiently estimated. 
\cite{arora2014provable} prove that certain \emph{very sparse} deep
neural networks with random connection patterns can be learned in polynomial time and sample
complexity. In model (\ref{eq:Model}), this would correspond to  
random, sparse vectors $\bw^{\ell}$, $\{u_{\ell}\}_{\ell=1}^{k}$. Their techniques do not seem applicable to non-random
connections or to non-sparse graphs.
Finally, there are several celebrated results on the sample complexity of approximating a function through a neural network (see, e.g., \cite{anthony2009neural}). This is a different problem than the \emph{parameter estimation} problem we solve here. A formal understanding of parameter estimation is  crucial in understanding why neural nets learn low-dimensional representations well. Parameter estimation also naturally arises in learning mixtures, where correctly identifying the constituent components (or, modes) is of equal or greater importance than regressing the mixture function. 

\section{Modelling Assumptions and Learning Algorithm}
\label{sec:Method}

%Our approach to learning the parameter vectors
%$\{\bw_\ell\}_{\ell=1}^{k}$ relies on estimating the gradient of
%the regression function \eqref{eq:Model} at several points in $\reals^d$. Before presenting our method, we give here some intuition on how such estimates can be used to discover the model parameters, as well as how the gradient can be estimated.

\subsection{Modeling Assumptions}\label{model}

Recall that we consider a regression model with response variable $y\in \reals$, generated through \eqref{eq:Model} from covariates $\bx\in \reals^d$. 
Note that \eqref{eq:Model} is equivalent to
%
%\begin{align}
%
$y =\sum_{\ell=1}^k u_\ell\, f(\<\bw^\ell,\bx\>) + \eps\, ,$
%
%\end{align}
%
with $\eps$ a noise term such that $\E\{\eps|\bx\} = 0$.  
In the general case, we assume that for any  $\ell\in[k]\equiv\{1,\ldots,k\}$, the vectors $\bw^\ell$ have
unit norm, i.e., $\|\bw^\ell\|_2=1$,  
 the absolute value of each weight is at most one, i.e., $u_\ell\in
 [-1,1]$,   and  the response is bounded, i.e.,
 $y\in [-M,M]$ for some $M>0$. In the mixture case, we assume in addition that $u_\ell\geq 0$, $l\in[k]$, and $\sum_{\ell=1}^ku_\ell=1.$ %Also, we assume the nonlinearity $f(\,\cdot\,)$ to be known.  We note  again that, if $\{\bw^{\ell}\}_{\ell=1}^{k}$ are known, we can use ordinary least squares to fit the coefficients $\{u_\ell\}_{\ell=1}^k$.
%We shall focus therefore on estimating $\bw^1,\dots,\bw^k$.
%\si{Regressing the $u$'s presumes that $\beta$ is known. This is the only place we presume knowledge of $\beta$. How should we pitch this?}

Clearly, we cannot learn a vector $\bw^\ell$ if $u_\ell=0$, nor
distinguish two vectors $\bw^\ell,\bw^{\ell'}$, where $\ell\neq
\ell'$, if they are identical. For this reason, we make the following two
additional assumptions. First,  coefficients $u_\ell$ are bounded away from
zero, i.e., there exists a $u_0$ s.t.~$0<u_0\leq| u_\ell|<1$, for all
$\ell\in[k]$. Second,  the collinearity between any subset of 
vectors $\bw^\ell$ is also bounded. Formally, let 
$\bM= [\bw^1|\bw^2|\cdots|\bw^k]\in \reals^{d\times k}$ be the matrix comprising the vectors $\bw^\ell$ column-wise. We assume that there exists a $\kappa>0$ such that 
$\kappa\le \sigma_{\rm min}(\bM)$, 
 where 
 $\sigma_{\rm min}$ the smallest
singular value of $\bM$.  Intuitively, the existence of  $\kappa$ implies a lower bound on the angle between any two vectors $\bw^\ell,\bw^{\ell'}$. %  (see, e.g., the proof of Lemma~\ref{anglelemma}).

Our learning method relies on producing estimates of the gradient $\nabla r(\cdot)$ at arbitrary points in $\reals^d$. We produce estimators of the gradient under two different models on our ability to sample function $r$:
\begin{packeditemize}
\item \textbf{Value Oracle Model}. Under our first model, we assume access to a \emph{value oracle}: given a $\bx\in \reals^d$, the oracle produces a $y\in \reals$ governed by \eqref{eq:Model}, while successive calls to the oracle are independent. Put differently, we treat \eqref{eq:Model} as a `black box', whose inputs are under our algorithm's control. %Our algorithm uses this oracle to construct estimates of $\nabla r(\bxi)$ at different $\bxi \in \reals^d$, by making $n_0$ oracle calls for each $\bxi$.
\item \textbf{Gaussian Covariates Model}. Under our second model, we assume that the covariates $\bx$ follow a standard Gaussian~$\normal(0,\id_{d\times d})$, while $y$ is given by \eqref{eq:Model}. Our learning algorithm has access to $n$ independent pairs $(\bx^{(1)},y^{(1)})$, \dots, $(\bx^{(n)},y^{(n)})$ generated from the above joint distribution, and must construct gradient estimates $\nabla r(\bxi)$ at different $\bxi \in \reals^d$ from this dataset alone.
\end{packeditemize}   

%$\Var(\eps|\bX) \le  \sigma^2$ uniformly  over $\bX$ {\bf [A: Need some
%  assumption like this.]}

\subsection{Intuition Behind our Approach}\label{sec:intuition}

Consider the gradient of the expected response function $r:\reals^d\to \reals$, evaluated at a $\bxi\in \reals^d$:
\begin{align}\label{gradient}
\nabla r(\bxi)=\textstyle\sum_{\ell=1}^k u_\ell  f'(\<\bw^\ell,\bxi\>) \cdot \bw^\ell\, 
\end{align}  
Observe that $\nabla r(\bxi)$ is a linear combination of the parameter vectors $\bw^\ell$. Moreover, since $f$ is a sigmoid, $\lim_{|t|\to\infty} f'(t) =0 $. Thus, for any $\bxi$ such that $|\<\bw^\ell, \bxi\>| \gg 1$, the coefficient $ u_\ell  f'(\<\bw^\ell,\bxi\>) $ weighing the contribution of $\bw^\ell$ to the gradient $\nabla r(\bxi)$ is small. As a result,  $\bw^\ell$ contributes significantly to the gradient $\nabla r(\bxi)$ when it is approximately normal to $\bxi$, i.e., $|\<\bw^\ell, \bxi\>| \approx 0$. 

These observations motivate our approach.  Presuming the existence of an estimator of the gradient, our algorithm amounts to the following three steps:
\begin{packedenumerate}
\item Pick several $\bxi \in \reals^d$, and produce estimates of the gradient $\bw(\bxi)\approx \nabla r (\bxi)$.
\item If $\|\bw(\bxi)\|_2$ is below a threshold $w_0$, ignore this estimate. Otherwise, normalize it, producing $\btw(\bxi) = \bw(\bxi)/\|\bw(\bxi)\|_2$. 
\item Identify $k$ clusters among the resulting `candidate' vectors
  $\btw(\bxi)$, and report the centers of these $k$ clusters as the parameter vector estimates for $\{\bw^\ell\}_{\ell=1}^k$. 
\end{packedenumerate}
If a $\bw(\bxi)$ has a high norm, then $\bxi$ must be approximately normal to at least one vector in  $\{\bw^\ell\}_{\ell=1}^k$. Moreover, if it is approximately normal to only one such vector, say $\bw^1$, by \eqref{gradient} the estimated gradient $\bw(\bxi)$ will have a significant component in the direction of $\bw^1$. As such, after renormalization, all such vectors are indeed clustered around $\bw^1$. Our formal guarantees, as stated by Theorem~\ref{maintheorem}, establish that  most candidates indeed satisfy this property, with the exception of a small spurious set.
\begin{comment}
To illustrate this, we plot in Figure~\ref{fig:vectorfield}  a vector field the gradient estimates produced through our gradient estimation method for $d=2$ at various $\bxi\in \reals^d$. Observe that the largest of these vectors indeed include directions parallel to the parameter vectors. Moreover, these directions manifest precisely at $\bxi$s that are normal to one of the parameter vectors.

\begin{figure*}[!t]
{\hspace*{\stretch{1}}\includegraphics[width=0.6\textwidth]{VectorField.pdf}\hspace*{\stretch{1}}}
\caption{\small Vector field representation of the gradient estimated by \eqref{eq:Slope} for $d=2$, under $n=20000$ i.i.d.~standard Gaussian covariates, with $\beta=1$, $u=[0.5,0.5]$, and the standard basis as vectors $\{\bw^\ell\}_{\ell=1,2}.$  }\label{fig:vectorfield}
\end{figure*}
\end{comment}

\subsection{Gradient Estimation}\label{sec:gradest}
The above approach crucially relies on estimating the response
gradient $\nabla r(\cdot)$ at an arbitrary $\bxi\in \reals^d$. We discuss our estimation process for each of the two models below.

\begin{packeditemize}
\item\noindent\textbf{Gradient Estimation in the Value Oracle Model.} Under the Value Oracle model, given a $\bxi\in \reals^d$, we generate $n_0$ i.i.d.~pairs $(\bx^{(i)},y^{(i)})$ where each $\bx^{(i)}$ is sampled from $\normal(\bxi,\id_{d\times d})$, and $y^{(i)}$ is the corresponding value returned by the oracle. The estimate  $\bw(\bxi)$ of  $\nabla r(\bxi)$ is then:
\begin{align}
\bw(\bxi) = \textstyle \frac{1}{n_0}\sum_{i=1}^{n_0} (\bx^{(i)}-\bxi)y^{(i)}. \label{voestimate}
\end{align}
\end{packeditemize}
If $m_0$ is the number of gradient estimates, the total number of oracle calls  is $n=m_0\times n_0$.

\begin{packeditemize}
\item\noindent\textbf{Gradient Estimation in the Gaussian Covariates Model.}
%One way to produce an estimate of $\nabla r(\bxi)$ in the Gaussian Covariates model to construct a linear approximation of $r$ locally around a
%point $\bxi$. This can be achieved by, e.g., regressing $r(\bX)$ from
%the covariates in a neighborhood around $\bxi$; alternatively, the
%regression can involve all samples $(\bx^{(i)},y^{(i)})$,
%$i=1,\ldots,n$, with samples weighed by a kernel that decays with the
%distance from $\bx$. Though such approaches appear promising in
%practice, we choose a simpler method to simplify our
%analysis. Namely, we estimate the gradient through an appropriately
%weighed correlation between the response and the covariates.
Given a $\bxi\in\reals^d$, we first compute
the `barycenter' of all covariates $\bx^{(i)}$ w.r.t.~the  exponential kernel
%
%\begin{align}
%
$K(\bxi,\cdot) = \exp\big\{\<\bxi,\cdot\>\big\}\, ,$ namely, 
%
%\begin{align}
%
$\bx(\bxi) = \frac{\sum_{i=1}^nK(\bxi,\bx^{(i)})\, \bx^{(i)}}{\sum_{i=1}^nK(\bxi,\bx^{(i)})}\, .$
%
%\end{align}
%
%
%
Then, we compute the estimate $\bw(\bxi)$ of  $\nabla r(\bxi)$  as: 
\begin{align}
\bw(\bxi) \equiv \frac{\sum_{i=1}^nK(\bxi,\bx^{(i)})\,
  y^{(i)}\big(\bx^{(i)}-\bx(\bxi)\big)}{\sum_{i=1}^nK(\bxi,\bx^{(i)})} \, . \label{eq:Slope}
\end{align}
\end{packeditemize}
Note that \emph{the same} $n$ covariate/label pairs $\{(\bx^{(i)},y^{(i)})\}_{i=1}^{n}$ are used in the computation of each estimate $\bw(\bxi)$. This is in contrast to the Value Oracle model, where inputs to \eqref{eq:Model} are centered on $\bxi$.

\noindent\textbf{Correctness}. The estimates $\bw(\bxi)$ produced under either of the two models through \eqref{voestimate} and \eqref{eq:Slope}, respectively, capture the local slope of the
regression function:   both constitute
 asymptotically unbiased estimates of $\E_\bxi[\nabla r(\bX)]$, namely,
 the expectation of the gradient when
 $\bX\sim\normal(\bxi,\id_{d\times d})$ (c.f.~Lemmas~\ref{assympboundsvoestimate} and~\ref{assympboundscor}). However, our algorithm (Algoritm~\ref{algo:Candidates}) is agnostic to how the gradient is estimated. %These estimates suffice for establishing the concentration of candidates around the parameter vectors, as illustrated by Theorem~\ref{maintheorem}; 
In principle as well as in practice,  alternative approaches (like, e.g., using different kernels, or regressing $r(\cdot)$ locally at $\bxi$ through linear approximation) could be used instead.

\subsection{Candidate Generation Algorithm}

 The entire candidate generation process  is summarized in
 Algorithm~\ref{algo:Candidates}. In short, we first produce 
 of $m_0$ i.i.d.~vectors $\bxi$, sampled from a common Gaussian distribution
 $p_{\bxi} = \normal(0,\xi_0^2\,\id_{d})$, with covariance proportional to the identity. For each such $\bxi$, we produce a
 gradient estimate $\bw(\bxi)$ using Eq.~\eqref{voestimate} or \eqref{eq:Slope}. We ignore all
 estimates whose norm is below a threshold, namely
 $\|\bw(\bxi)\|_2\le w_0$. Finally, we
 normalize the remaining estimates, thus producing  the final `candidate'
 set 
$\bw(\bxi^{1})$, \dots, $\bw(\bxi^{m})$, where $m\leq m_0$.
%Throughout our analysis, we use $p_{\bxi} = \normal(0,\xi_0^2\,\id_{d})$, i.e., a
%centered  Gaussian distribution  with covariance proportional to the
%identity.
Both $\bxi_0$ as well as the threshold $w_0$ are design parameters, which we specify below in our convergence theorem (Theorem~\ref{maintheorem}).

Fig.~\ref{fig:execution} illustrates an execution of Algorithm~\ref{algo:Candidates}. The candidates generated are indeed close to the parameter vectors, which are succesfully recovered through simple $k$-means over these candidates.

\begin{algorithm}[!t]
\caption{\textsc{CandidateGeneration}}
\label{algo:Candidates}
\begin{small}
\begin{algorithmic}
%\STATE {\bf Input:} feature vector--label pairs:
%$\{(\bx^{(i)},y^{(i)})\}_{1\le i\le n}$
\STATE $\ell \leftarrow 0$
\FOR{$i\in\{1,2,\dots,m_0\}$}
\STATE generate $\bxi\sim p_{\bxi}$; compute $\bw(\bxi)$ using Eq.~\eqref{voestimate} or Eq.~(\ref{eq:Slope})\label{estimateline}
\IF{$\|\bw(\bxi)\|_2\ge w_0$}
\STATE $\ell\leftarrow \ell+1$; $\bxi^{\ell} \leftarrow \bxi$; $\btw(\bxi^{\ell}) \leftarrow\bw(\bxi^{\ell})/\|\bw(\bxi^{\ell})\|_2$
\ENDIF 
\ENDFOR
\STATE $m\leftarrow \ell$;
 \textbf{return} $m$ and $\{\btw(\bxi^1),\dots,\btw(\bxi^m)\}$
\end{algorithmic}
\end{small}
\end{algorithm}

\begin{figure*}[!t]
\includegraphics[width=0.325\textwidth]{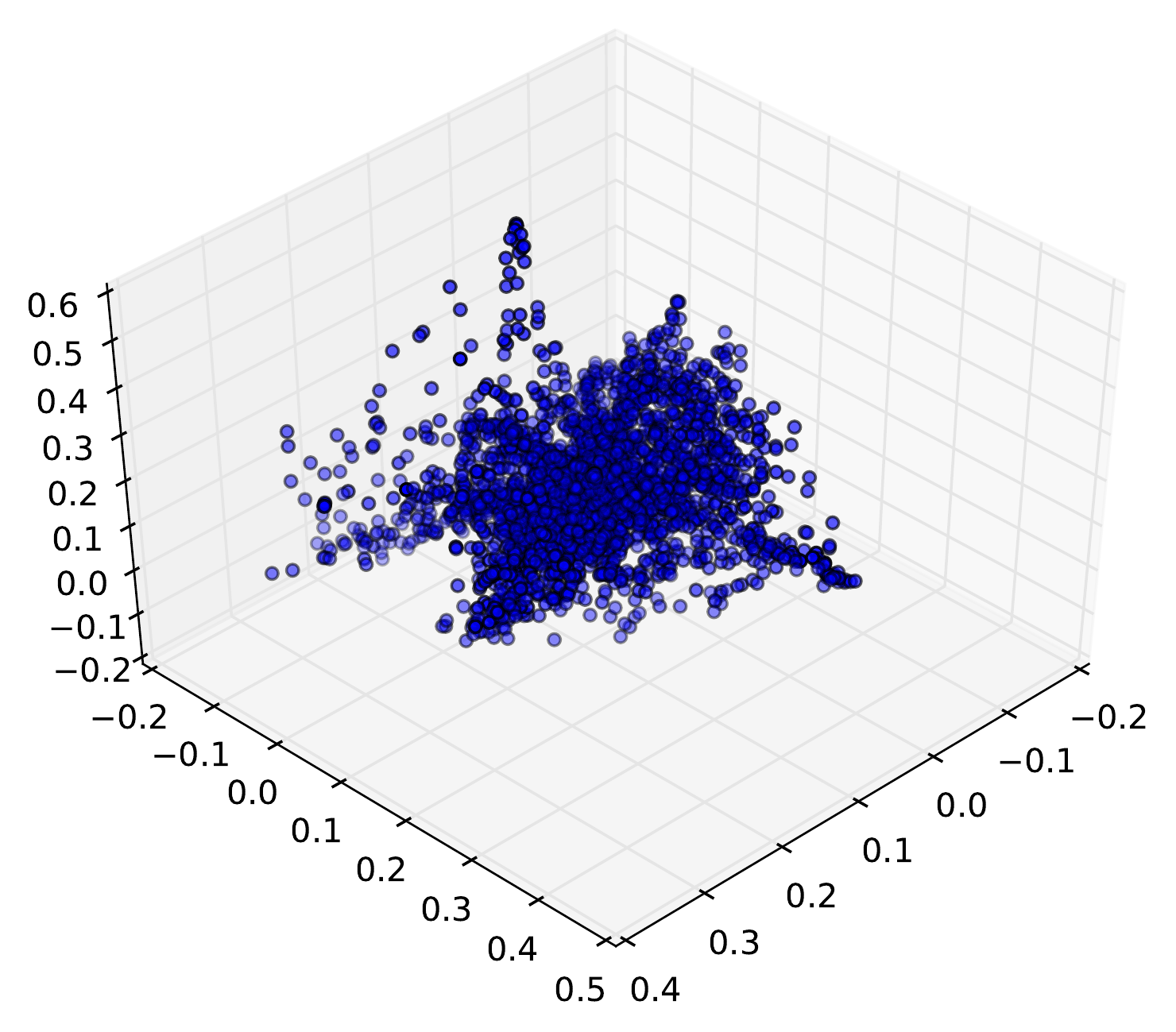}
\includegraphics[width=0.325\textwidth]{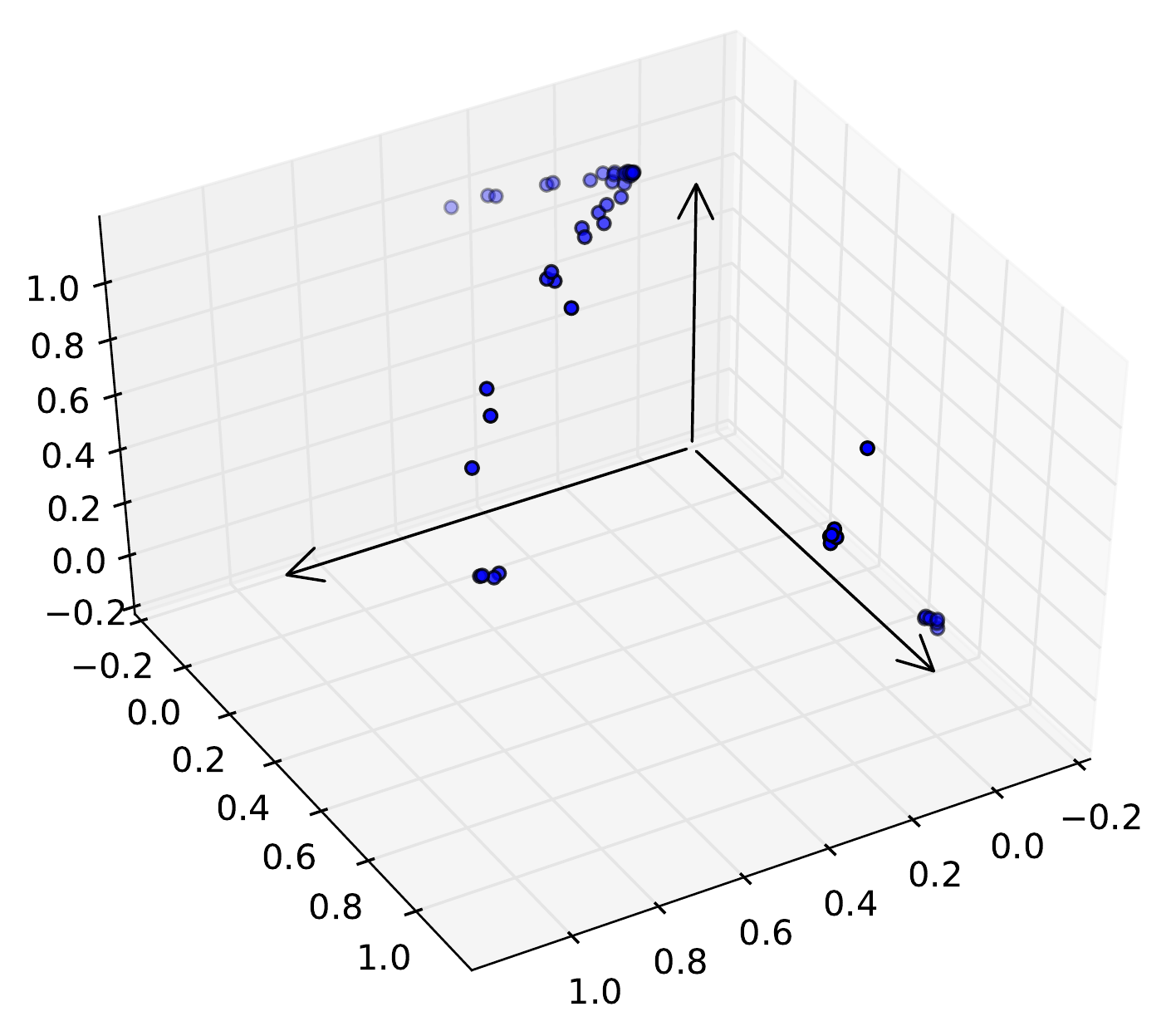}
\includegraphics[width=0.325\textwidth]{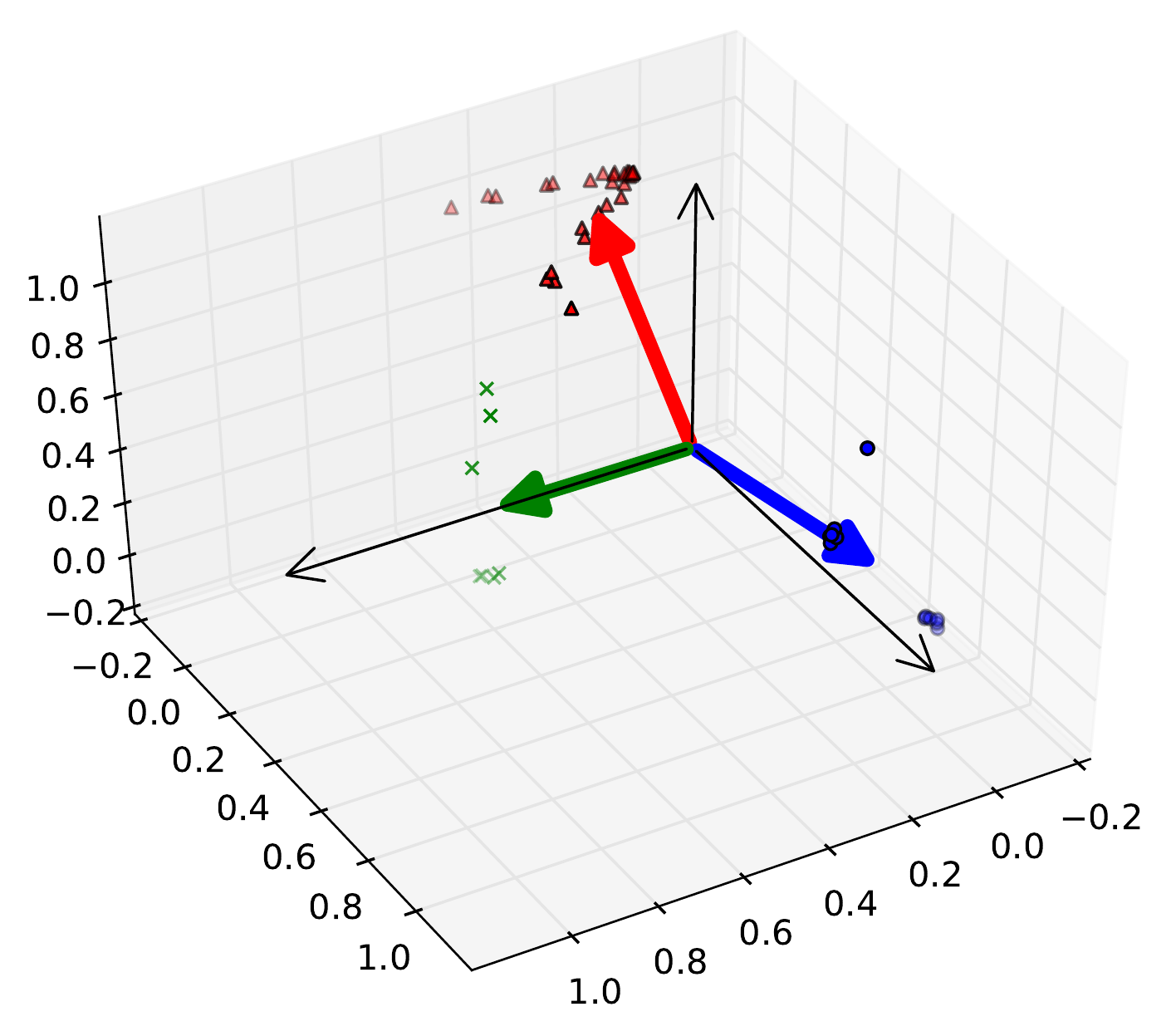}\\
\hspace*{\stretch{1}}(a)\hspace{\stretch{2}}(b)\hspace{\stretch{2}}(c)\hspace{\stretch{1}}
\caption{\small An execution of Algorithm~\ref{algo:Candidates} under the Gaussian Covariates model, for $d=3$, $n=m_0=5000$, the standard basis as parameter vectors, and  uniform weights, using  $\xi_0 = 2\sqrt{d}$. The  $m$ pre-candidate vectors $\bw(\bxi)$ are shown in subfigure (a) as points. Subfigure (b) contains the top 50 vectors with highest norm, after they have been normalized. These vectors are clustered  using $k$-means in subfigure (c), and the resulting cluster centers are indicated with thick arrows. }\label{fig:execution}
%x [0.2132885 ]    [-0.05875299]  [0.96253497]
%y [ 0.08618616]   [ 0.90684679]  [0.36175821]
%z [ 0.73388837 ]  [ 0.42372956]  [0.48975551]
%
%
\end{figure*}

\section{Main Results}\label{sec:Mainresult}

\subsection{Generic Combinations under Value Oracle and Gaussian Covariates Models}
Our first result establishes that Algorithm~\ref{algo:Candidates} indeed produces  candidates  clustered around  $\{\bw^\ell\}_{\ell=1}^k$: %Indeed, our theorem below shows that, excluding a spurious set of candidates, the remaining candidates concentrate around the parameter vectors for large enough $m_0$ and $n_0$ or $n$ (for the Value Oracle or Gaussian Covariates models, respectively). 
\begin{theorem}\label{maintheorem}
Let  $\{\btw(\xi^1),\dots,\btw(\xi^m)\}$, where $m\leq m_0$, be the output of
the \textsc{CandidateGeneration} algorithm. Then, for any $\delta \in (0,0.5]$, and any $\rho\in(0,1)$,
there exist $w_0 = \Theta\left(\frac{\delta u_0^2}{k}\right) $, $\xi_0= \Theta\left(\frac{k^2}{\kappa\beta\rho}\log\frac{k}{u_0\delta} \right) $,  $\gamma=\Theta(\frac{\kappa\rho}{k^2})$, for which
the following occur
with probability at least $1-\delta$: the set of candidate indices
 $\cC \subseteq [m_0]$  can be partitioned as
%
%\begin{align}
%
$\cC = \cC_0 \cup\cC_1\cup \cdots\cup\cC_k\, $ %\label{eq:PartitionC0Cl}
%
%\end{align}
%
so that
  if  $i\in\cC_\ell$  then
%
%\begin{align}
%
%i\in \cC_\ell \;\;\Rightarrow \;\;
$ \big\|\btw(\bxi^{i}) -\bw^{\ell}\big\|_2\le 6 \delta\, ,$
%
%\end{align}
and $|\cC_\ell|\geq \gamma m_0/2$, for all
$\ell\in[k]$,
while $\cC_0$ is a set of `bad' candidates such that
 $|\cC_0|\le 2\rho \gamma m_0$.
This occurs 
    for (a)
%\begin{align}
$m_0>C\frac{1}{\gamma\rho}\log\frac{k}{\delta}\, ,$ %\label{eq:M0LowerBound}
%\end{align}
and (b) either
%\begin{align}
$n_0> C'\frac{dM^2}{\delta^2w_0^2} \log\frac{k}{\gamma \rho \delta}\, ,$
%\end{align}
 under the Value Oracle model, or
%\begin{align}
$n > \frac{M^4 d^2}{\delta^2w_0^2}\max\left(C'' \Big(\frac{k}{\gamma\rho\delta}\right)^ {1 +7(1+2d^{-1})\xi_0^2}, e^{  4d \left(\frac{1}{7}+   (1+2d^{-1}) \xi_0^2\right)  } \Big)\, , $
%\end{align}
under the Gaussian Covariates model, for some absolute constants $C,C',C''$.
%$n> \frac{C M^4 d^2k}{\delta^3w_0^2\gamma\rho}
%e^{(2d+4)\xi_0^2}$ 

\end{theorem}

There are several observations to be made. First, the gradient estimation procedures outlined in \ref{sec:gradest} indeed yield sufficiently accurate estimates so that, asymptotically, 
the candidates $\bw^{\ell}$ concentrate around the parameter vectors. On the other hand, there exists also a set of `bad' candidates, that may not necessarily be close to any parameter vector.
However,  this set can be made \emph{arbitrarily
small} compared to the smallest set of `valid' candidates: indeed, 
for any $\rho\in (0,1)$, choosing $m_0$ and $n_0$ or $n$ as in the theorem yields
$|\cC_0|/\min_{\ell}|\cC_\ell|\leq 4\rho$. 
The  spurious set $\cC_0$ is unavoidable; beyond incorrect estimates the gradient (occurring with low probability as $n_0$ and 
$n$ increases), if $d>2$, there are also $\bxi$s that are
approximately normal to \emph{more} than one parameter vectors
$\bw^\ell$, $\ell\in [k]$. Nonetheless, this 
 is significantly  less likely than the event that $\bxi$ is approximately normal to (and thus, estimating) only one of the
parameter vectors (c.f.~Lemma~\ref{anglelemma}).

In both gradient estimation models, a small ($m_0=\Theta(\log k)$) number of $\bxi$s suffices to correctly estimate the clusters. On the other hand, the sample complexity scales as $n=m_0\times n_0 = \Theta(d\log k)$  in the Value Oracle model, and $n= e^{\Theta(d \cdot \mathtt{poly}(k))}$ in the Gaussian Covariates model. Nevertheless, in the next section, we show that the exponential dependence on $d$ can be avoided when $k\ll d$. % through a dimensionality reduction method.
\subsection{Gaussian Covariates with Dimensionality Reduction}

We avoid the exponential dependence on $d$ under Gaussian Covariates by preprocessing covariates through a dimensionality reduction method. Observe that $r(\bx)$ depends on $\bx$ only through the $k$ inner products $\<\bw^\ell,\bx\>$, $\ell \in[k]$. As a result, projecting a $\bx^{(i)}$ to the $k$-dimensional linear space spanned by $\{\bw^{\ell}\}_{\ell=1}^k$ and estimating the gradient would result in no loss of information. Most importantly, this   eliminates the dependence of the gradient estimation on $d$. 
Discovering the linear span of  $\{\bw^{\ell}\}_{\ell=1}^k$  can be performed using, e.g., the \textsc{SpectralMirror} method of \cite{sun2014learning} in the  mixture case:
\begin{theorem}[\cite{sun2014learning}]\label{sunthm}
 Given $n$ covariate/label pairs  generated through the Gaussian Covariates model in the mixture case, the \textsc{SpectralMirror} Algorithm constructs an estimate $\hat{\mathcal{M}}$ of $\mathcal{M}=\linspan(\bM)$ s.t., for all $\theta\leq u_0\kappa$, the largest principal angle $d_P$ between $\mathcal{M}$ and $\hat{\mathcal{M}}$ satisfies
%\begin{align}
$\prob\{d_P(\mathcal{M},\hat{\mathcal{M}})>\theta\}\leq C_1e^{-C_2 \frac{n\theta^2}{d}}.$
%\end{align}
with $C_1$, $C_2$ absolute constants. 
\end{theorem}
The above theorem  is a special case of Theorem~1 of \citep{sun2014learning}, when the covariates are sampled from a standard (rather than arbitrary) Gaussian. In short, it implies that $O(\frac{d}{\theta^2})$ points suffice to produce a linear space within a $\theta$ angle from $\mathcal{M}=\linspan(\bM)$.
We leverage this result to improve on the bound of Theorem~\ref{maintheorem} for Gaussian covariates.
First, given $n$ samples from the Gaussian Covariates model, we use as subset of these samples to produce an estimate $\hat{\mathcal{M}}$ of $\mathcal{M}$ through the \textsc{SpectralMirror} algorithm.  
 To  estimate  the gradient $\nabla r(\bxi)$ at $\bxi\in\reals^d$  using the remaining samples, we apply again \eqref{eq:Slope} \emph{on the projections of $\bx^{(i)}$ to $\hat{\mathcal{M}}.$}
More formally:
\begin{packeditemize}
\item\textbf{Gradient Estimation with Projected Gaussian Covariates.} Use $n_1<n$ samples to produce an estimate $\hat{\mathcal{M}}$ of $\mathcal{M}$.
   For every $\bx\in\reals^{d}$, denote by $\bhx$ the projection of $\bx$ to $\hat{\mathcal{M}}$. Using the remaining $n_2= n-n_1$ samples, the estimate of $\nabla r(\bxi)$ at $\bxi\in\reals^s$ is given by: %, the estimate of $\nabla r(\bxi)$ is then given by:
\begin{align}\label{projectedestimate}
 \bw(\bxi) \equiv \frac{\sum_{i=n_1+1}^nK(\bxi,\bhx^{(i)})\,
 y^{(i)}\big(\bhx^{(i)}-\bhx(\bxi)\big)}{\sum_{i=n_1+1}^nK(\bxi,\bhx^{(i)})}, 
\end{align}
where $K(\bxi,\bx)=e^{\<\bxi,\bx\>}$ and
%\begin{align}
%
$\bhx(\bxi) \equiv \frac{\sum_{i=n_1+1}^nK(\bxi,\bhx^{(i)})\, \bhx^{(i)}}{\sum_{i=n_1+1}^nK(\bxi,\bhx^{(i)})} .$%
%\end{align}
\end{packeditemize}
Note that 
$\bw(\bxi) = \bw(\bhxi)$, i.e., the estimate depends on $\bxi$ only through its projection to $\hat{\mathcal{M}}$.
The  estimation \eqref{projectedestimate} replaces \eqref{eq:Slope} in Algorithm~\ref{algo:Candidates}, indeed eliminating the exponential dependence on $d$:
\begin{theorem}\label{dimredtheorem}
 Let  $\{\btw(\xi^1),\dots,\btw(\xi^m)\}$, where $m\leq m_0$, be the output of
the \textsc{CandidateGeneration} algorithm when \eqref{projectedestimate} is used to produce $\bw(\bxi)$ in the mixture case. Then, for any $\delta \in (0,0.5]$, and any $\rho\in(0,1)$,
there exist $w_0 = \Theta\left(\frac{\delta u_0^2}{k}\right) $, $\xi_0= \Theta\left(\frac{k^2}{\kappa\beta\rho}\log\frac{k}{u_0\delta} \right) $,  $\gamma=\Theta(\frac{\kappa\rho}{k^2})$, for which
the following occurs with probability at least $1-\delta$: the set of candidate indices
 $\cC \subseteq [m_0]$,  can be partitioned as
%
%\begin{align}
%
$\cC = \cC_0 \cup\cC_1\cup \cdots\cup\cC_k\, ,$ %\label{eq:PartitionC0Cl}
%
%\end{align}
%
so that
  if  $i\in\cC_\ell$  then
%
%\begin{align}
%
%i\in \cC_\ell \;\;\Rightarrow \;\;
$ \big\|\btw(\bxi^{i}) -\bw^{\ell}\big\|_2\le 7 \delta\, ,$
%
%\end{align}
and $|\cC_\ell|\geq \gamma m_0/2$, for all
$\ell\in[k]$,
while $\cC_0$ is a set of `bad' candidates such that
 $|\cC_0|\le 2\rho \gamma m_0$.
This occurs 
    for (a)
%\begin{align}
$m_0>C\frac{1}{\gamma\rho}\log\frac{k}{\delta}\, ,$ %\label{eq:M0LowerBound}
%\end{align}
and (b) %either
%\begin{align}
%n_0> C'\frac{dM^2}{\delta^2w_0^2} \log\frac{k}{\gamma \rho \delta}\, ,
%\end{align}
%if gradient estimates are computed under the Value Oracle model, or
%\begin{align}
$n_1>C'd\left(\min \left\{ 2\arcsin \frac{\sqrt{3}\kappa}{8k },2\arcsin \frac{\delta^2}{4k}, u_0\kappa \right\}\right)^{-1}$
%\end{align}
and
%\begin{align}
$n_2 > \frac{M^4 k^2}{\delta^2w_0^2}\max\left(C'' \left(\frac{k}{\gamma\rho\delta}\right)^ {1 +7(1+2k^{-1})\xi_0^2}, e^{  4k \left(\frac{1}{7}+   (1+2k^{-1}) \xi_0^2\right)  } \right), $
%\end{align}
 for  absolute constants $C,C',C''$.
%$n> \frac{C M^4 d^2k}{\delta^3w_0^2\gamma\rho}
%e^{(2d+4)\xi_0^2}$ 

\end{theorem}
Hence,  $n_1=\Theta(dk)$ suffices to estimate $\mathcal{M}$, while $n_2=e^{\Theta(k^5\log^2k)}$  suffices for gradient estimation.

\begin{comment}
Finally, as mentioned in the introduction, the sample complexity
scales exponentially in the dimension $d$. We believe that this
behavior can improved to a polynomial dependence 
by a better choice of the kernel $K(\,\cdot\, ,\,\cdot\,)$ and by dimensionality reduction techniques, like the one by \cite{sun2014learning}.
On the other hand, the computational complexity is significantly smaller 
because of the following remarks.
%
%\begin{enumerate}
%
%\item 
First, the number $m_0$ of centers $\bxi^{\ell}$ at which the gradient
  is estimated is very small, see Eq.~(\ref{eq:M0LowerBound}). Indeed
  it is \emph{independent of the dimension} and only quadratic in the
  number of hidden units.
%
%\item 
Second, for each of these centers, the gradient is estimated using
  Eq.~(\ref{eq:Slope}). While formally, this has complexity
  $\Theta(n)$, in reality only points close to $\xi$ contribute
  significantly. For instance,  the kernel $K(\,\cdot\,,\,\cdot\,)$
  could be replaced with a kernel with strictly finite range, at the
  price of a somewhat more complicate proof. Indeed, $O(d)$ pints
  should be sufficient to estimate the gradient at $\bxi$.
%
%\end{enumerate}
%
This suggests that roughly $k^2d$ data points are actually used by our
algorithm, i.e. the computational complexity is only polynomial in the
number of hidden units and problem dimensions.
\end{comment}
%
%******************************
%

\section{Proof of Theorem \ref{maintheorem}}
\label{sec:Proof}

We prove Theorem~\ref{maintheorem} below. The proofs of all lemmas, as well as the proof of Theorem~\ref{dimredtheorem}, can be found in the \withappendix{appendix.}{supplement.}
%Throughout, we will write $K = \Kexp$ for the exponential kernel
%(dropping the subscript for simplicity).

\subsection{Concentration Results}
%We use lower case for deterministic quantities, and upper case for  
%corresponding random variables.
We first establish some concentration results regarding gradient estimation through \eqref{voestimate} and \eqref{eq:Slope}. %In fact, both estimators $\bw(\bxi)$ concentrate around the expectation of the gradient centered at $\bxi$. In particular,  
For $\bxi\in \reals^d$, let  $\E_\bxi\{\cdot\}$ be the expectation with respect to a Gaussian random variable
 $\bX\sim\normal(\bxi,\id_{d\times d})$ centered at $\bxi$, and let:
 %Then, the estimates $w(\bxi)$ under both the Value Oracle and the Gaussian Covariates models concentrate around:
\begin{align}\label{population}
\bow(\bxi) = \E_\bxi \big\{\nabla r(\bX)\big\}  =\textstyle\sum_{\ell=1}^k
u_\ell\bw^\ell\E_{\bxi}\big\{f'\big(\<\bw^\ell,\bX\>\big)\big\}\, .
\end{align}
Given $\bxi\in \reals^d$,  \eqref{population} is the expectation of estimate $\bw(\bxi)$, under both \eqref{voestimate} and \eqref{eq:Slope}: this is a consequence of Stein's lemma \cite{stein1973estimation}.
We also characterize the rate of convergence of $\bw(\bxi)$ to $\bow(\bxi)$:
\begin{lemma}[Value Oracle Concentration Bound]\label{assympboundsvoestimate} There exist numerical constants $c_1,c_2,c_2,$ and $c_4$ such that, when $\bw(\bxi)$ is computed through \eqref{voestimate}, for any fixed $\bxi\in\reals^d$:
\begin{align}
\label{voebound}
\prob\Big\{\big\|\bw(\bxi)-\bow(\bxi)\big\|_2\ge \delta\Big\}&\leq c_1
\exp\Big(-\min \big\{ \frac{c_2n_0\delta^2 }{d M^2} , (c_3\frac{\sqrt{n_0}\delta}{M} - c_4 \sqrt{d})^2\big\}\Big)\, .
\end{align}
\end{lemma}
\withappendix{The proof of this lemma can be found in Appendix~\ref{app:concentration}, and relies on the sub-gaussianity of the r.v.~$y\bx$, when $\bx$ is gaussian and $y$ is given by  \eqref{eq:Model}.}{}
%This is precisely the expectation of the gradient over a Gaussian centered at $\bxi$. Indeed:
Similarly, under the Gaussian Covariates model:
\begin{lemma}[Gaussian Covariates Concentration Bound]\label{assympboundscor}
%For any fixed $\bxi\in\reals^d$, let $\E_0\{\cdots\}$ denote the
%expectation 
%with respect to $\bX\sim\normal(0,\id_{d})$. Then we have
%
%\begin{align}\label{population}
%%
%\bow(\bxi) = \sum_{\ell=1}^k
%u_\ell\bw^\ell\E_{0}\big\{f'\big(\<\bw^\ell,\bxi+\bX\>\big)\big\}\, .
%%
%\end{align}
%
There exists a numerical constant $C$ such that,
when $\bw(\bxi)$ is computed through \eqref{eq:Slope}, for any fixed $\bxi\in\reals^d$:
\begin{align}
\prob\Big\{\big\|\bw(\bxi)-\bow(\bxi)\big\|_2\ge \delta\Big\}&\leq
\frac{Ce^{\|\bxi\|_2^2}}{n\delta^2}\, M^4(d+\|\bxi\|^2)^2\, .
\label{candidatebound}
\end{align}
\end{lemma}
\withappendix{The proof of this lemma can also be found in Appendix~\ref{app:concentration}.}{}

\subsection{Characterizing Gradient Coefficients and Approximate Normality}

Eq.~\eqref{population} indicates that, asymptotically, $\bow(\bxi)$ is a linear combination of the vectors $\bw^\ell$. The following lemma\withappendix{, proved in Appendix~\ref{proofofcoeffbounds},}{} bounds the coefficients of this linear combination:
\begin{lemma}\label{lemma:coeffboundssimple}For any $\ell\in[k]$ and $\bxi\in \reals^d$,  
%$ z_\ell\equiv \< \bw^\ell,\bxi \>$.
% Then
\begin{align}
\beta\Phi(-2\beta)\, 
e^{-2\beta| \< \bw^\ell,\bxi \>|+2\beta^2}\leq\E_{\bxi}\big\{f'\big(\<\bw^\ell,\bX\>\big)\big\} \leq 
  {8}{\beta} e^{-2\beta|  \< \bw^\ell,\bxi \>|+2{\beta^2}} \, ,
\label{lowerupperbound}  
\end{align}
%\begin{align}  
%\E_{0}&\big\{f'\big(\<\bw^i,\bxi+\bX\>\big)\big\}\ge
%  \, ,
%\label{lowerbound}
%\end{align}
%
where $\Phi(x) = \int_{-\infty}^x e^{-z^2/2} \de z/\sqrt{2\pi}$ is the one-dimensional
Gaussian distribution function, and $\E_\bxi$ is the expectation with respect to a Gaussian random variable
 $\bX\sim\normal(\bxi,\id_{d\times d})$ centered at $\bxi$.
\end{lemma}
The lemma implies that a vector $\bw^\ell$ contributes significantly to  $\bow(\bxi)$ only if $z_\ell = \<\bw^\ell,\bxi \> \simeq 0$, and $\bxi$ is approximately normal to $\bw^\ell$. Thus, if $\bxi$ is approximately normal to \emph{only one} $\bw^\ell$, $\bw(\bxi)\approx \bw^\ell$.
Clearly, the success of the candidate generation process depends on the event that a randomly generated $\bxi$ is on approximately normal to a single parameter vector, but not two. The following lemma\withappendix{, whose proof can be found in Appendix~\ref{proofofanglelemma},}{} bounds the probabilities of these events:
%Armed with this result, we can characterize the probability that a generated $\bxi$ is normal to only a single parameter vector as follows:
\begin{lemma}\label{anglelemma} Assume that $\bXi\in \reals^d$ is sampled from $\normal(0,\xi_0^2\,\id_{d})$. Then, for any $0<\Delta<\xi_0$,
$\prob(|\< \bw^\ell,\bXi\>|<\Delta)\geq \sqrt{\frac{2}{e\pi}}\frac{\Delta}{\xi_0},$ for all $\ell\in[k] $
and for any $\Delta_1,\Delta_2>0$, 
$\prob(|\<\bw^\ell,\bXi\>|<\Delta_1,|\<\bw^{\ell'},\bXi\>|<\Delta_2)
\leq \frac{2\Delta_1\Delta_2}{\pi\kappa \xi_0^2 }$, for all $\ell,\ell'\in[k]$ with $\ell\neq\ell'$.
\end{lemma}

%\subsection{Completing the Proof of Theorem~\ref{maintheorem}}.
\subsection{Candidate Partitioning}\label{sec:partitioning}
We now describe how the $m$
candidate indices
 $\cC \subset [m_0]$ produced by Algorithm~\ref{algo:Candidates}  can be partitioned as
$\cC = \cC_0 \cup\cC_1\cup \cdots\cup\cC_k\, ,$ 
s.t.~for any $i\in \cC_\ell$, candidate $\btw(\bxi^{(i)})$ is close to $\bw^\ell$, while $\cC_0$ is a small set of spurious candidates.
Let $c_1=c_1(\beta) \equiv\beta\Phi(-2\beta)e^{2\beta^2} $ and $c_2=c_2(\beta)
\equiv  8\beta e^{2\beta^2}$, where $\Phi$ as in Lemma \ref{lemma:coeffboundssimple}.  %, so that Lemma \ref{lemma:coeffboundssimple} reads 
%\begin{align}c_1\,
%e^{-2\beta|\<\bw^\ell,\bxi\>|}\leq \E_{\bxi}\big\{f'\big(\<\bw^\ell,\bX\>\big)\big\} \le
%c_2\,e^{-2\beta|\<\bw^\ell,\bxi\>|}\label{simplebded},\quad\text{for }\ell=1,\ldots,k.\end{align}
%Throughout the proof we use $C,C_1,C_2$ and so on, to denote constants that
%may depend on $\beta$.
%\begin{proof}
Given $\delta\in(0,0.5],$ and $\rho\in(0,1)$, let 
\withappendix{%
\begin{align}\Delta = \ \frac{1}{\beta}\log \frac{(1+\delta) c_2k}{c_1u_0 \delta } ,\label{Delta} \end{align}%
}
{%
$\Delta = \ \frac{1}{\beta}\log \frac{(1+\delta) c_2k}{c_1u_0 \delta } ,\label{Delta}$%
}
and set the parameters of Algorithm~\ref{algo:Candidates} as follows
\begin{align}w_0 &\equiv \frac{1}{\delta} k c_2 e^{-2\beta \Delta} =   \frac{ c_1^2u_0^2\delta}{(1+\delta)^2c_2k}, \label{w0}\\
\xi_0 &\equiv 2\sqrt{\frac{2e}{\pi}} \frac{k}{\kappa} (\frac{k}{\rho} + 1) \Delta =  \frac{2}{\beta} \sqrt{\frac{2e}{\pi}} \frac{k}{\kappa} (\frac{k}{\rho}+1) \log\frac{(1+\delta) c_2k}{c_1 u_0\delta },\label{xi0}\quad\text{and}\\
\gamma & \equiv \sqrt{\frac{1}{2e\pi}} \frac{\Delta}{\xi_0} -\frac{2k}{\kappa\pi} \left(\frac{\Delta}{\xi_0}\right)^2 = \frac{\kappa\rho }{4e(\rho+\kappa)^2}\label{gamma}
\end{align}
Note that our choice of $\xi_0$ is such that $\frac{\Delta}{\xi_0}$ satisfies the equation:
\begin{align}
\frac{2 k^2}{\pi\kappa}
\left(\frac{\Delta}{\xi_0}\right)^2=\rho\cdot\left(\sqrt{\frac{1}{2e\pi}} \frac{\Delta}{\xi_0} -\frac{2k}{\kappa\pi} \left(\frac{\Delta}{\xi_0}\right)^2\right)=\rho \gamma. \label{rhoeq}
\end{align}
We define the following partition of $\reals^d=\cR_0 \bigcup
\Big\{\bigcup_{\ell=1}^{k}\cR_{\ell}\Big\}\cup \cR_*\, ,$:
\begin{subequations}\label{Rsets}
\begin{align}
\cR_0 &\equiv \big\{\bxi\in\reals^d:\, \min_{i\in [k]}
|\<\bw^{i},\bxi\>|\ge \Delta\big\}\, ,\\
\cR_{\ell} & \equiv \big\{\bxi\in\reals^d:\, \,
|\<\bw^{\ell},\bxi\>|<\Delta,\;
\min_{i\in [k]\setminus \ell}|\<\bw^{i},\bxi\>|\ge\Delta\big\}\, ,\\
\cR_{*} & \equiv \big\{\bxi\in\reals^d:\, \,\exists \ell_1,\ell_2\in
[k]:\; \ell_1\neq\ell_2,\;
|\<\bw^{\ell_1},\bxi\>|<\Delta,\;
|\<\bw^{\ell_2},\bxi\>|<\Delta\big\}\, .
\end{align}
\end{subequations}
%
%We then make the following claims:
%
%\begin{enumerate}
%\item 
By  \eqref{population} and \eqref{lowerupperbound}, for $\bxi\in \cR_0$,  $\bow(\bxi)$ can
be rewritten as $\bow(\bxi) = \bM\bv$, where  $\|\bv\|_{2}\le \sqrt{k}c_2e^{-2\beta\Delta}$. Hence, as $\|\bM\|_2\leq\|\bM\|_F=\sqrt{k}$, for any $\bxi\in \cR_0$, 
\begin{align}
\|\bow(\bxi)\|_2\le k\, c_2\, e^{-2\beta\Delta} \stackrel{\eqref{w0}}{=} \delta w_0\, .\label{eq:InR0}
\end{align}
%
%The first inequality above follows from Eq.~(\ref{population}) and \eqref{simplebded}.  
Similarly, the sets $\cR_\ell$ are such that for any $\bxi\in \cR_{\ell}$%
\begin{align}
\|\bow(\bxi)-a_{\ell}\bw^{\ell}\|_2\le k\, c_2\, e^{-2\beta\Delta} = \delta w_0\, ,\label{allbutone}
\end{align}
where $a_\ell \equiv |u_{\ell}|\cdot \E_0\Big\{f'\big(\<\bw^\ell,\xi+\bX\>\big)\Big\}$.
This follows from the same argument used above in proving \eqref{eq:InR0}.
Moreover, from Eq.~(\ref{lowerupperbound}): 
\begin{align}
c_1u_{0}e^{-2\beta\Delta}\le
|a_{\ell} |= |u_{\ell}|\cdot \E_0\Big\{f'\big(\<\bw^\ell,\xi+\bX\>\big)\Big\}\le
c_2\, .\label{albounds}
\end{align}
%\end{enumerate}
%
Armed with the above observations, we partition the set of generated $\bxi$'s  as $[m_0]\equiv \cG\cup \cG^c$, where
%
%\begin{align}
%
$
\cG  \equiv \Big\{j\in [m_0]:\, \|\bw(\bxi_j)-\bow(\bxi_j)\|_2\le
\delta w_0\Big\}\, .
$
%
%\end{align}
%
Recall that the candidate set is, by construction, $\cC \equiv \Big\{j\in [m_0]:\, 
\|\bw(\bxi_j)\|_2\ge w_0\}$.
We define the partition of the candidate set, as described in Theorem~\ref{maintheorem}, as follows: for each $\ell\in [k]$, let
%
%\begin{align}
%
$
\cC_{\ell} \equiv \Big\{j\in \cG:\, \bxi_j\in \cR_\ell, \;
\|\bw(\bxi_j)\|_2\ge w_0\Big\}\, ,
$
%
%\end{align}
%
and 
%
%\begin{align}
%
$
\cC_0 \equiv \Big\{j\in [m_0]:\, 
\|\bw(\bxi_j)\|_2\ge w_0\, , j\not\in\cup_{\ell=1}^k\cC_{\ell}\Big\}\,.
$
%
%\end{align}
%
Observe that this is indeed a partition of $\cC$. 
%Next we claim that:
%
%\begin{itemize}
%
The following lemma\withappendix{, whose proof can be found in Appendix~\ref{proofofclusterlemma},}{} establishes that candidates in the sets $\cC_\ell$ have the desirable property stated in Thm.~\ref{maintheorem}, namely, that they are clustered around the corresponding vectors $w_\ell$:
\begin{lemma}\label{clusterlemma}
 For each $\ell\in [k]$ and each $j\in\cC_{\ell}$,
  $\big\|\btw(\bxi_{j}) -\bw^{\ell}\big\|_2\le 6\delta$.
\end{lemma}
%\end{itemize}
%
%In order to prove $(i)$,

%\subsection{Bounding the Size of Each Partition}
To conclude the proof, we need to show that, w.h.p., the sets $\cC_\ell$ are large, while the spurious set $\cC_0$ is small. 
The next lemma upper-bounds the size of the spurious candidate set $C_0$: %, under the Value Oracle model (i.e., when the estimates $\bw(\bxi)$ are generated through \eqref{voestimate}) and under the Gaussian Covariates model (i.e., when estimates are generated through \eqref{eq:Slope}), respectively. 
\begin{lemma}\label{spuriouslemma}  The event $|\cC_0|\le  2 \gamma \rho m_0 $ occurs  with probability  at least (b) $1-  \Big[\frac{c_1}{\gamma\rho}
\exp\Big(-\min \big\{ \frac{c_2n_0\delta^2w_0^2 }{d M^2} , (c_3\frac{\sqrt{n_0}\delta w_0}{M} - c_4 \sqrt{d})^2\big\}\Big)+ e^{-c_5m_0\gamma \rho}\Big]$, with $c_1,\ldots,c_5$  absolute constants, under the Value Oracle model, and (b)  $1 - \Big(
\frac{c_1}{   \gamma \rho  } \left(\frac{M^4 d^2}{n\delta^2w_0^2} \right)^{\frac{1}{1 +4(1+2d^{-1})\xi_0^2}} 
%\frac{C_1M^4 d^2}{n\delta^2w_0^2\gamma\rho}
%e^{(2d+4)\xi_0^2}
 +  e^{-c_2 m_0 \gamma\rho}\Big)$, for $n>  \frac{M^4 d^2}{\delta^2w_0^2} e^{  4d  \left(\frac{1}{7}+   (1+2d^{-1}) \xi_0^2\right)  }$ and $c_1, c_2$ absolute constants,  under the Gaussian Covariates model.\end{lemma}
\withappendix{The proof can be found in Appendix~\ref{proofofspurious}.}{} 
%Next, consider $\cB_*$.
%
% Hence we have that
%\begin{align*}
%$
%\prob\Big(|\cC_0|>2m_0\gamma\rho)\leq \prob\big(|\cG^c|\ge \frac{m_0\gamma\rho}{2}\big)+\prob\left( |\cB_*|\ge m_0
%\gamma\rho\frac{3}{2}\right) 
%\leq %\frac{2CM^4 d^2}{n\delta^2w_0^2\gamma\rho}
%e^{(2d+4)\xi_0^2} 
%
% \frac{2C}{   \gamma \rho  } \left(\frac{M^4 d^2}{n\delta^2w_0^2} \right)^{\frac{1}{1 +7(1+2d^{-1})\xi_0^2}}
%\, 
%+  e^{- m_0 \gamma\rho/16}.
%$
%   \end{align*} 
The next lemma\withappendix{, whose proof is in Appendix~\ref{proofofgoodclusterlemma}}{} lower-bounds the size of sets $\cC_{\ell}$:
\begin{lemma} \label{goodclusterlemma}
For $\ell\in[k]$, the event $|\cC_{\ell}|\ge m_0\gamma/2 $  occurs  with probability at least (a)
$1-  \Big[\frac{c_1}{\gamma\rho}
\exp\Big(-\min \big\{ \frac{c_2n_0\delta^2w_0^2 }{d M^2} , (c_3\frac{\sqrt{n_0}\delta w_0}{M} - c_4 \sqrt{d})^2\big\}\Big)+  e^{-c_5m_0\gamma}\Big]$, where $c_1,\ldots,c_5$ are absolute constants, under the Value Oracle model,  and (b) $1-\left(
\frac{c_1}{   \gamma \rho  } \left(\frac{M^4 d^2}{n\delta^2w_0^2} \right)^{\frac{1}{1 +7(1+2d^{-1})\xi_0^2}}
%\frac{C_3M^4 d^2}{n\delta^2w_0^2\gamma\rho}
%e^{(2d+4)\xi_0^2}
+ e^{-c_3m_0\gamma}\right)$, where $c_1,c_2$ are absolute constants, for  $n>  \frac{M^4 d^2}{\delta^2w_0^2} e^{   4d\left(\frac{1}{7}+   (1+2d^{-1}) \xi_0^2\right)  }$,  under the Gaussian Covariates model.  \end{lemma}
Using the above three lemmas and by applying a union bound, we get that the events in the theorem occur with probability at least $1-\delta$ if 
$m_0>C\frac{1}{\gamma\rho}\log\frac{k}{\delta}$
and, for the Value Oracle model,
$n_0> C'\frac{dM^2}{\delta^2w_0^2} \log\frac{k}{\gamma \rho \delta}, $
or, for the Gaussian Covariates model,
$n > \frac{M^4 d^2}{\delta^2w_0^2}\max\left(C'' \left(\frac{k}{\gamma\rho\delta}\right)^ {1 +7(1+2d^{-1})\xi_0^2}, e^{  4d \left(\frac{1}{7}+   (1+2d^{-1}) \xi_0^2\right)  } \right), 
 $
where $C$,$C'$, and $C''$ are absolute constants.\qed

\bibliographystyle{abbrvnat}
\bibliography{all-bibliography}

\begin{thebibliography}{25}
\providecommand{\natexlab}[1]{#1}
\providecommand{\url}[1]{\texttt{#1}}
\expandafter\ifx\csname urlstyle\endcsname\relax
  \providecommand{\doi}[1]{doi: #1}\else
  \providecommand{\doi}{doi: \begingroup \urlstyle{rm}\Url}\fi

\bibitem[Anandkumar et~al.(2012)Anandkumar, Huang, Hsu, and
  Kakade]{anandkumar2012learning}
A.~Anandkumar, F.~Huang, D.~J. Hsu, and S.~M. Kakade.
\newblock Learning mixtures of tree graphical models.
\newblock In \emph{NIPS}, 2012.

\bibitem[Anandkumar et~al.(2014)Anandkumar, Ge, Hsu, and
  Kakade]{anandkumar2014tensor}
A.~Anandkumar, R.~Ge, D.~Hsu, and S.~M. Kakade.
\newblock A tensor approach to learning mixed membership community models.
\newblock \emph{The Journal of Machine Learning Research}, 15\penalty0
  (1):\penalty0 2239--2312, 2014.

\bibitem[Andoni et~al.(2014)Andoni, Panigrahy, Valiant, and
  Zhang]{andoni2014learning}
A.~Andoni, R.~Panigrahy, G.~Valiant, and L.~Zhang.
\newblock Learning polynomials with neural networks.
\newblock In \emph{ICML}, 2014.

\bibitem[Anthony and Bartlett(2009)]{anthony2009neural}
M.~Anthony and P.~L. Bartlett.
\newblock \emph{Neural Network Learning: Theoretical Foundations}.
\newblock Cambridge University Press, 2009.

\bibitem[Arora et~al.(2014)Arora, Bhaskara, Ge, and Ma]{arora2014provable}
S.~Arora, A.~Bhaskara, R.~Ge, and T.~Ma.
\newblock Provable bounds for learning some deep representations.
\newblock In \emph{ICML}, 2014.

\bibitem[Barron(1993)]{barron1993universal}
A.~R. Barron.
\newblock Universal approximation bounds for superpositions of a sigmoidal
  function.
\newblock \emph{Information Theory, IEEE Transactions on}, 39\penalty0
  (3):\penalty0 930--945, 1993.

\bibitem[Bengio(2009)]{bengio2009learning}
Y.~Bengio.
\newblock Learning deep architectures for {AI}.
\newblock \emph{Foundations and Trends in Machine Learning}, 2\penalty0
  (1):\penalty0 1--127, 2009.

\bibitem[Boureau et~al.(2008)Boureau, Cun, et~al.]{boureau2008sparse}
Y.-l. Boureau, Y.~L. Cun, et~al.
\newblock Sparse feature learning for deep belief networks.
\newblock In \emph{NIPS}, 2008.

\bibitem[Boureau et~al.(2010)Boureau, Bach, LeCun, and
  Ponce]{boureau2010learning}
Y.-L. Boureau, F.~Bach, Y.~LeCun, and J.~Ponce.
\newblock Learning mid-level features for recognition.
\newblock In \emph{CVPR}, 2010.

\bibitem[Chaganty and Liang(2013)]{chaganty2013spectral}
A.~T. Chaganty and P.~Liang.
\newblock Spectral experts for estimating mixtures of linear regressions.
\newblock In \emph{ICML}, 2013.

\bibitem[Chen et~al.(2014)Chen, Yi, and Caramanis]{chen2014convex}
Y.~Chen, X.~Yi, and C.~Caramanis.
\newblock A convex formulation for mixed regression with two components:
  Minimax optimal rates.
\newblock In \emph{COLT}, 2014.

\bibitem[Dasgupta and Gupta(2003)]{dasgupta2003elementary}
S.~Dasgupta and A.~Gupta.
\newblock An elementary proof of a theorem of johnson and lindenstrauss.
\newblock \emph{Random Structures \& Algorithms}, 22\penalty0 (1):\penalty0
  60--65, 2003.

\bibitem[Dempster et~al.(1977)Dempster, Laird, and Rubin]{dempster1977maximum}
A.~P. Dempster, N.~M. Laird, and D.~B. Rubin.
\newblock Maximum likelihood from incomplete data via the em algorithm.
\newblock \emph{Journal of the Royal Statistical Society. Series B
  (Methodological)}, pages 1--38, 1977.

\bibitem[Hinton et~al.(2012)Hinton, Deng, Yu, Dahl, Mohamed, Jaitly, Senior,
  Vanhoucke, Nguyen, Sainath, et~al.]{hinton2012deep}
G.~Hinton, L.~Deng, D.~Yu, G.~E. Dahl, A.-r. Mohamed, N.~Jaitly, A.~Senior,
  V.~Vanhoucke, P.~Nguyen, T.~N. Sainath, et~al.
\newblock Deep neural networks for acoustic modeling in speech recognition: The
  shared views of four research groups.
\newblock \emph{Signal Processing Magazine, IEEE}, 29\penalty0 (6):\penalty0
  82--97, 2012.

\bibitem[Hsu and Kakade(2013)]{hsu2013learning}
D.~Hsu and S.~M. Kakade.
\newblock Learning mixtures of spherical {Gaussians}: moment methods and
  spectral decompositions.
\newblock In \emph{ITCS}, 2013.

\bibitem[Humphrey et~al.(2013)Humphrey, Bello, and LeCun]{humphrey2013feature}
E.~J. Humphrey, J.~P. Bello, and Y.~LeCun.
\newblock Feature learning and deep architectures: new directions for music
  informatics.
\newblock \emph{Journal of Intelligent Information Systems}, 41\penalty0
  (3):\penalty0 461--481, 2013.

\bibitem[Krizhevsky et~al.(2012)Krizhevsky, Sutskever, and
  Hinton]{krizhevsky2012imagenet}
A.~Krizhevsky, I.~Sutskever, and G.~E. Hinton.
\newblock Imagenet classification with deep convolutional neural networks.
\newblock In \emph{NIPS}, 2012.

\bibitem[Liu(1994)]{liu1994Siegel}
J.~S. Liu.
\newblock Siegel's formula via {Stein}'s identities.
\newblock \emph{Statistics \& Probability Letters}, 21\penalty0 (3):\penalty0
  247--251, 1994.

\bibitem[Mairal et~al.(2009)Mairal, Ponce, Sapiro, Zisserman, and
  Bach]{mairal2009supervised}
J.~Mairal, J.~Ponce, G.~Sapiro, A.~Zisserman, and F.~R. Bach.
\newblock Supervised dictionary learning.
\newblock In \emph{NIPS}, pages 1033--1040, 2009.

\bibitem[Moitra and Valiant(2010)]{moitra2010settling}
A.~Moitra and G.~Valiant.
\newblock Settling the polynomial learnability of mixtures of gaussians.
\newblock In \emph{Foundations of Computer Science (FOCS), 2010 51st Annual
  IEEE Symposium on}, pages 93--102. IEEE, 2010.

\bibitem[Sedghi and Anandkumar(2016)]{sedghi2014provable}
H.~Sedghi and A.~Anandkumar.
\newblock Provable tensor methods for learning mixtures of classifiers.
\newblock In \emph{AISTATS}, 2016.

\bibitem[Stein(1973)]{stein1973estimation}
C.~M. Stein.
\newblock Estimation of the mean of a multivariate normal distribution.
\newblock In \emph{Prague Symposium on Asymptotic Statistics}, 1973.

\bibitem[Sun et~al.(2014)Sun, Ioannidis, and Montanari]{sun2014learning}
Y.~Sun, S.~Ioannidis, and A.~Montanari.
\newblock Learning mixtures of linear classifiers.
\newblock In \emph{ICML}, 2014.

\bibitem[Yi et~al.(2014)Yi, Caramanis, and Sanghavi]{icml2014c2_yia14}
X.~Yi, C.~Caramanis, and S.~Sanghavi.
\newblock Alternating minimization for mixed linear regression.
\newblock In \emph{ICML}, 2014.

\bibitem[Yu et~al.(2013)Yu, Seltzer, Li, Huang, and Seide]{yu2013feature}
D.~Yu, M.~L. Seltzer, J.~Li, J.-T. Huang, and F.~Seide.
\newblock Feature learning in deep neural networks-studies on speech
  recognition tasks.
\newblock \emph{arXiv preprint arXiv:1301.3605}, 2013.

\end{thebibliography}

\appendix

\section{Proof of Concentration Results }\label{app:concentration}

\subsection{Proof of Lemma~\ref{assympboundsvoestimate}}

We  use the following variant of Stein's identity (see \cite{stein1973estimation}, and \cite{liu1994Siegel} for this specific formulation). Let  $\bX\in \reals^d$, $\bX'\in \reals^{d'}$ be jointly Gaussian random vectors, sampled from a Gaussian distribution of arbitrary mean and covariance. Consider a function \mbox{$h:\reals^{d'}\to\reals$} that is almost everywhere (a.e.)~differentiable and satisfies $\E[|\partial h(\bX')/\partial x_i|]<\infty$, for all $i\in [d']$. Then, the following identity holds:
\begin{align}\mathtt{Cov}(\bX,h(\bX')) = \mathtt{Cov}(\bX,\bX')\E[\nabla h(\bX') ]. \label{stein}\end{align}

Observe first that for $\bX$ a Gaussian vector centered at $\bxi$:
\begin{align}
\mathtt{Cov}(\bX,r(\bX))&\equiv \E_{\bxi}\left[  (\bX-\E[\bX]) (r(\bX)-\E[r(\bX)])\right] =  \E_\bxi\left[  (\bX-\E[\bX]) r(\bX)\right] \nonumber\\
&= \E_\bxi\left[  (\bX-\bxi) r(\bX)\right]. \label{dropmean}\end{align}
Thus, in the case of the Value Oracle model, we have that:
\begin{align*}
\E[ \bw(\bx)] &\stackrel{\eqref{voestimate}}{=} \E_{\bxi}[ (\bX-\bxi) r(\bX)] \stackrel{\eqref{dropmean}}{=} \mathtt{Cov}(\bX,r(\bX)) \\
&\stackrel{\eqref{stein}}{=} \mathtt{Cov}(\bX,\bX) \E_{\bxi}[\nabla r(\bX)]= \E_{\bxi}[\nabla r(\bX)]\equiv \bow(\bxi).\\
\end{align*}
Thus, $\bw(\bxi)$ indeed concentrates around $\bow(\bxi)$ by the law of large numbers. The tail bounds in Lemma~\ref{assympboundsvoestimate} then follow from Lemma~1 of \cite{sun2014learning}. \qed

\subsection{Proof of Lemma~\ref{assympboundscor}}
It is convenient to define the following quantities
%
%\begin{align*}
%
$z(\bxi)  \equiv \sum_{i=1}^n K(\bxi,\bx^{(i)})\, ,$ 
$\bu(\bxi) \equiv \sum_{i=1}^nK(\bxi,\bx^{(i)})\, \bx^{(i)}\, ,$
$s(\bxi) \equiv \sum_{i=1}^nK(\bxi,\bx^{(i)})\, y^{(i)}\, , $ and 
$\bv(\bxi) \equiv \sum_{i=1}^nK(\bxi,\bx^{(i)})\, y^{(i)}\bx^{(i)}\, .$
%
%\end{align*}
%
Note that, in terms of  these quantities, we have
%
%\begin{align}
%
$\bw(\bxi)  = \frac{\bv(\bxi)}{z(\bxi)} -
\frac{\bu(\bxi) \, s(\bxi)}{z(\bxi)^2}\, . $ 
%
%\end{align}
%
%In the definition below (following our convention) we use capitals to
%denote random variables corresponding to the above definitions.
The following concentration results then hold:
\begin{lemma}\label{expectationsandbounds}
For any fixed $\bxi\in\reals^d$, let $\E_{\bxi}\{\cdots\}$ denote
the expectation 
with respect to $\bX\sim\normal(\bxi,\id_{d})$. Then, if $\{x^{(i)}\}_{i=1,\ldots,n}$ are generated under the Gaussian covariates model, we have
\begin{subequations}\label{expectations}
\begin{align}
\E\, z(\bxi) & = n \, e^{\|\bxi\|_2^2/2}\, , &
\E\, \bu(\bxi) & = n\, e^{\|\bxi\|_2^2/2}\, \E_{\bxi} \bX\, =  n\, e^{\|\bxi\|_2^2/2}\, \bxi,\label{zexpuexp}\\
\E\, s(\bxi) & = n\, e^{\|\bxi\|_2^2/2}\, \E_{\bxi} r(\bX)\, , &
\E\, \bv(\bxi) & = n\, e^{\|\bxi\|_2^2/2}\, \E_{\bxi}\big\{\bX\,r(\bX)\big\}\, .\label{sexpvexp}
\end{align}
\end{subequations}
and %if $\|\xi\|_2\le \sqrt{c\log n}$ for $c =???$, then
\begin{subequations}\label{bounds}
\begin{align}
&\prob\Big\{\big|z(\bxi)\!-\!\E\, z(\bxi)\big|\ge n\delta\Big\}\!\le\! \frac{e^{2\|\bxi\|_2^2}}{n\delta^2} \, , 
\prob\Big\{\big\|\bu(\bxi)\!-\!\E\, \bu(\bxi)\big\|_2\ge n\delta\Big\} \!\le\!  \frac{e^{2\|\bxi\|_2^2}(d\!+\!4\|\bxi\|_2^2)}{n\delta^2} \label{zboundubound},\\
&\prob\Big\{\big|s(\bxi)\!-\!\E\, s(\bxi)\big|\ge n\delta\Big\}\!\le\! \frac{e^{2\|\bxi\|_2^2}M^2}{n\delta^2} \, ,  
\prob\Big\{\big\|\bv(\bxi)\!-\!\E\, \bv(\bxi)\big\|_2\ge n\delta\Big\}\!\le\!  \frac{e^{2\|\bxi\|_2^2}M^2(d\!+\!4\|\bxi\|_2^2) }{n\delta^2}\label{sboundvbound}  .
\end{align}
\end{subequations}
%\begin{align}
%
%\prob\Big\{\big|z(\bxi)-\E\, z(\bxi)\big|\ge n\delta\Big\}&\le 2\,
%e^{-n\delta^2/2}\, ,\\
%\prob\Big\{\big\|\bU(\bxi)-\E\, \bU(\bxi)\big\|_2\ge n\delta\Big\}&\le 2\,
%e^{-n\delta^2/2}\, ,\\
%\prob\Big\{\big|s(\bxi)-\E\, s(\bxi)\big|\ge n\delta\Big\}&\le 2\,
%e^{-n\delta^2/2}\, ,\\
%\prob\Big\{\big\|\bV(\bxi)-\E\, \bV(\bxi)\big\|_2\ge n\delta\Big\}&\le 2\,
%e^{-n\delta^2/2}\, ,
%
%\end{align}
%
\end{lemma}
%

%\subsection{Proof of Lemma~\ref{expectationsandbounds}}
\begin{proof}
We use the following two simple properties of Gaussian random variables.
For $g:\reals^d\to \reals^m$, we have that for $\bX\sim\normal(0,\id_{d})$:
\begin{align}\E_0[e^{\bxi^T\bX} g(\bX)] = e^{\|\bxi\|_2^2/2} \E_{\bxi}[ g(\bX)] \label{expectation}\end{align}
and
\begin{align}\Cov[e^{\bxi^T\bX} g(\bX)] = e^{2\|\bxi\|_2^2} \E_{2\bxi}[g(\bX)g^T(\bX)] -e^{\|\bxi\|_2^2}\E_{\bxi}[g(\bX)]\E_{\bxi}[g^T(\bX)]\label{covariance}
\end{align}
The statements in \eqref{expectations} therefore follow from \eqref{expectation} and the definition of the kernel $K$. By Chebyshev's inequality,  %Applying \eqref{covariance} with  $g(\bX)=1$, we get
\begin{align*}
\prob\Big\{\big|z(\bxi)-\E\, z(\bxi)\big|\ge n\delta\Big\}\leq \frac{\Var\{e^{\bxi^T\bX}\}}{n\delta^2}\stackrel{\eqref{covariance}}{=} \frac{e^{2\|\bxi\|_2^2}-e^{\|\bxi\|_2^2}}{n\delta^2}\leq \frac{e^{2\|\bxi\|_2^2}}{n\delta^2}.
\end{align*}
Moreover, by Markov's inequality:
\begin{align*}
\prob\Big\{\big\|\bu(\bxi)&-\E\, \bu(\bxi)\big\|_2\ge n\delta\Big\} \leq\frac{\E \big\|\bu(\bxi)-\E\, \bu(\bxi)\big\|_2^2 }{n^2\delta^2}= \sum_{j=1}^d\frac{\Var\{ \bu_j(\bxi)\} }{n^2\delta^2} = \sum_{j=1}^d\frac{\Var\{ e^{\bxi^T \bX}\bX_j\} }{n\delta^2}  \\
%&=\sum_{j=1}^d \prob\Big\{\big|\bu_j(\bxi)-\E\, \bu_j(\bxi)\big|\ge \frac{n\delta}{\sqrt{d}}\Big\}\leq  d\sum_{j=1}^d\frac{\Var\{e^{\bxi^T\bX}\bX_j  \}}{n\delta^2}\\
&\stackrel{\eqref{covariance}}{\leq}\sum_{j=1}^d \frac{  e^{2\|\bxi\|_2^2} \E_{2\bxi}[\bX_j^2] }{n\delta^2}= \sum_{j=1}^d \frac{  e^{2\|\bxi\|_2^2} (1+4\bxi_j^2) }{n\delta^2} =\frac{ e^{2\|\bxi\|_2^2} (d+4\|\bxi\|_2^2)}{n\delta^2}
\end{align*}
The first two inequalities in \eqref{bounds} therefore follow. The remaining two follow similarly using the fact that the absolute values of the responses $y$ are bounded by $M$.  \end{proof}

An immediate consequence of   Lemma~\ref{expectationsandbounds} is that $\bw(\bxi)$ concentrates around the following quantity:
\begin{align*}
 \frac{\E\bv(\bxi)}{\E z(\bxi)} -
\frac{\E\bu(\bxi) \, \E s(\bxi)}{\big(\E z(\bxi)\big)^2}\, &=\E_{\bxi}\big\{\bX\,r(\bX)\big\}-\E_{\bxi}\bX \E_{\bxi}r(\bX)\\&=\Cov_{\bxi}[\bX,r(\bX)]\stackrel{\eqref{stein}}{=}\E_{\bxi}[\nabla r(\bX) ] \equiv \bow(\bxi).
\end{align*}
%
%\subsection{Proof of Corollary~\ref{assympboundscor}}
%\begin{proof}
%From \eqref{expectations} in, we have that:
%\begin{align*}
%% 
%\bow(\bxi)=\E_{\bxi}\big\{\bX\,r(\bX)\big\}-\E_{\bxi}\bX \E_{\bxi}r(\bX)=\Cov_{\bxi}[\bX,r(\bX)]= \E_{\bxi}[\nabla r(\bX) ],
%\end{align*}
Hence, \eqref{population} indeed describes the estimates, asymptotically.
To prove \eqref{candidatebound}, we use the following simple auxiliary lemma.
\begin{lemma}\label{fractionlemma}
For any $a,\bar{a}\in \reals^d$, $b,\bar{b}>0$, and $\delta > 0$, we have that:
\begin{align*}
\text{If }\|a-\bar{a}\|\leq \delta'\text{ and }|b-\bar{b}|\leq \delta'\text{ then }\left\|\frac{a}{b}-\frac{\bar{a}}{\bar{b}}\right\|\leq \delta 
\end{align*} 
where $$\delta' = \frac{\bar{b}^2\delta}{(\|\bar{a}\|+\bar{b}+\bar{b}\delta)}.$$
\end{lemma}
\begin{proof}(Sketch)
Note that $\delta'<\bar{b}$. It is easy to show that $\left\|\frac{a}{b}-\frac{\bar{a}}{\bar{b}}\right\| \leq \frac{(\bar{b}+\|\bar{a}\|)\delta'}{(\bar{b}-\delta')\bar{b}}=\delta$.
\end{proof}
We have that
\begin{align*}
&\prob\Big\{\big\|\bw(\bxi)-\bow(\bxi)\big\|_2>\delta\Big\} \leq \\& \prob\Big\{\left\|\frac{\bv(\bxi)}{ z(\bxi)} - \frac{\E\bv(\bxi)}{\E z(\bxi)}\right\|_2>\delta/2\Big\} + \prob\Big\{\left\|\frac{\bu(\bxi)s(\bxi)}{ z(\bxi)^2} -\frac{\E \bu(\bxi)\E s(\bxi)}{ (\E z(\bxi))^2} \right\|_2>\delta/2\Big\}
\end{align*}
From Lemma~\ref{fractionlemma}, for 
\begin{align}
\delta'=  \frac{(\E z(\bxi))^2 \delta}{2(  \|\E \bv(\bxi) \|_2+\E z(\bxi)+\E z(\bxi)\delta/2)} \stackrel{\eqref{expectations}}{=} \frac{ n \, e^{\|\bxi\|_2^2/2} }{2\|\E_{\bxi}\bX r(\bX)\|_2+2+\delta }\delta, \label{deltaprime}
\end{align}
we have
\begin{align*}
\prob\Big\{&\left\|\frac{\bv(\bxi)}{ z(\bxi)} - \frac{\E\bv(\bxi)}{\E
    z(\bxi)}\right\|_2>\delta/2\Big\} \leq
\prob\Big\{\|\bv(\bxi)-\E\bv(\bxi) \|_2>\delta' \Big\} + \prob\Big\{|
z(\bxi) -\E z(\bxi) |  >\delta'\Big\}\\
&\stackrel{\eqref{sboundvbound},\eqref{zboundubound},\eqref{deltaprime}}{\leq} \frac{M^2e^{2\|\bxi\|_2^2}(d+4\|\bxi\|_2^2)(2\|\E_{\bxi}\bX r(\bX)\|_2+2+\delta   )^2 }{n\delta^2 e^{\|\bxi\|_2^2}} + \frac{e^{2\|\bxi\|_2^2} (2\|\E_{\bxi}\bX r(\bX)\|_2+2+\delta   )^2 }{n\delta^2 e^{\|\bxi\|_2^2} }\\
& %= \frac{\left(M^2 (d+4\|\bxi\|_2^2) +1\right)\left(2M\sqrt{\E_{\bxi}\|\bX\|_2^2}+2+\delta  \right)^2  e^{\|\bxi\|_2^2}  }{n\delta^2}= 
=\frac{\left(M^2 (d+4\|\bxi\|_2^2) +1\right)\left(2M\sqrt{d+\|\bxi\|_2^2}+2+\delta  \right)^2  e^{\|\bxi\|_2^2}  }{n\delta^2},
\end{align*}
where in the second to last step we use $\|\E_{\bxi}\bX r(\bX)\|_2^2\leq\E_{\bxi}\|\bX r(\bX)\|_2^2 $, by the convexity of $\|\cdot\|_2^2$.
Similarly, for 
\begin{align}\delta'' &=  \frac{(\E z(\bxi))^4 \delta}{2(  \|\E \bu(\bxi)\|_2|\E s(\bxi)|+(\E z(\bxi))^2+(\E z(\bxi))^2\delta/2)} %\nonumber\\
\stackrel{\eqref{expectations}}{=} \frac{ n^2 \, e^{\|\bxi\|_2^2} }{2\|\E_{\bxi}\bX\|_2 |\E_{\bxi}r(\bX)|+2+\delta }\delta, \label{deltaprimeprime}
\end{align} we have
\begin{align*}
\prob\Big\{\left\|\frac{\bu(\bxi)s(\bxi)}{ z(\bxi)^2} -\frac{\E \bu(\bxi)\E s(\bxi)}{ (\E z(\bxi))^2} \right\|_2>\delta/2\Big\}&\leq \prob\Big\{\left\|\bu(\bxi)s(\bxi) -\E \bu(\bxi)\E s(\bxi) \right\|_2>\delta''\Big\} \\&\qquad+ \prob\Big\{\left| (z(\bxi))^2 -(\E z(\bxi))^2 \right|  >\delta''\Big\}
\end{align*}
Rewriting terms and applying a union bound gives
\begin{align}
\prob\Big\{&\left\|\bu(\bxi)s(\bxi) -\E \bu(\bxi)\E s(\bxi) \right\|_2>\delta''\Big\}\label{splitproducts}
\leq \prob\Big\{\left\|\E\bu(\bxi)\left(s(\bxi) -\E s(\bxi)\right) \right\|_2>\delta''/3\Big\}+\\&\qquad \prob\Big\{\left\|\E s(\bxi)\left(\bu(\bxi) -\E \bu(\bxi)\right) \right\|_2>\delta''/3\Big\}+
 \prob\Big\{\left\|\left(s(\bxi)-\E s(\bxi)\right)\left(\bu(\bxi) -\E \bu(\bxi)\right) \right\|_2>\delta''/3\Big\}\nonumber\end{align}
where
\begin{align*}
\prob\Big\{&\left\|\E\bu(\bxi)\left(s(\bxi) -\E s(\bxi)\right) \right\|_2>\delta''/3\Big\}
 = \prob\Big\{|s(\bxi) -\E s(\bxi)|>\frac{\delta''}{3\left\|\E\bu(\bxi)\right\|_2   }\Big\}\\
& \stackrel{\eqref{deltaprimeprime},\eqref{zexpuexp}}{=}   \prob\Big\{|s(\bxi) -\E s(\bxi)|>\frac{n \, e^{\|\bxi\|_2^2/2}    \delta   }{3\left\|\bxi\right\|_2 (2\|\E_{\bxi}\bX\|_2 |\E_{\bxi}r(\bX)|+2+\delta )  }\Big\}\\
& \stackrel{\eqref{sboundvbound}}{\leq}  \frac{9e^{\|\bxi\|_2^2}M^2\left\|\bxi\right\|_2^2 (2\|\E_{\bxi}\bX\|_2 |\E_{\bxi}r(\bX)|+2+\delta )^2    }{n\delta^2}
 \leq  \frac{9e^{\|\bxi\|_2^2}M^2\left\|\bxi\right\|_2^2 (2M\|\bxi\|_2 +2+\delta )^2    }{n\delta^2},
\end{align*}
\begin{align*}
\prob\Big\{&\left\|\E s(\bxi)\left(\bu(\bxi) -\E \bu(\bxi)\right) \right\|_2>\delta''/3\Big\}
 = \prob\Big\{\|\bu(\bxi) -\E \bu(\bxi)\|_2>\frac{\delta''}{3\left|\E s(\bxi)\right|   }\Big\}\\
& \stackrel{\eqref{deltaprimeprime},\eqref{sexpvexp}}{=}   \prob\Big\{\|\bu(\bxi) -\E \bu(\bxi)|_2>\frac{n \, e^{\|\bxi\|_2^2/2}    \delta   }{3|\E_{\bxi}r(\bX)| (2\|\E_{\bxi}\bX\|_2 |\E_{\bxi}r(\bX)|+2+\delta )  }\Big\}\\
& \stackrel{\eqref{zboundubound}}{\leq}  \frac{9e^{\|\bxi\|_2^2}(d+4\|\bxi\|_2^2)M^2 (2M\|\bxi\|_2 +2+\delta )^2    }{n\delta^2},
\end{align*}
and
\begin{align*}
\prob\Big\{&\left\|\left(s(\bxi)-\E s(\bxi)\right)\left(\bu(\bxi) -\E \bu(\bxi)\right) \right\|_2>\delta''/3\Big\}\\
&\leq   \prob\Big\{\left |s(\bxi)-\E s(\bxi)\right\|>\sqrt{\delta''/3}\Big\} +     \prob\Big\{\left\|\bu(\bxi) -\E \bu(\bxi) \right\|_2>\sqrt{\delta''/3}\Big\}\\
&\stackrel{\eqref{bounds},\eqref{deltaprimeprime}}{\leq} %\frac{3e^{2\|\bxi\|_2^2}M^2}{n\delta''} +  \frac{3e^{2\|\bxi\|_2^2}(d+4\|\bxi\|_2^2)}{n\delta''}\\
\frac{3e^{\|\bxi\|_2^2}(M^2+d+4\|\bxi\|_2^2) (2M\|\bxi\|_2+2+\delta   )  }{n\delta}.
\end{align*}
Finally, using a similar union bound as in  \eqref{splitproducts} we get:
\begin{align*}
\prob\Big\{\left| (z(\bxi))^2 -(\E z(\bxi))^2 \right|  >\delta''\Big\}& = \prob\Big\{\left| (z(\bxi) -\E z(\bxi))^2+2\E z(\bxi) (z(\bxi)-\E z(\bxi)) \right|  >\delta''\Big\}\\
&\leq  \prob\Big\{\left| z(\bxi) -\E z(\bxi)\right|>\sqrt{\frac{\delta''}{2}}\Big\}+\prob\Big\{\left |z(\bxi)-\E z(\bxi) \right|  >\frac{\delta''}{4\E z(\bxi)} \Big\} \\
&\stackrel{\eqref{bounds},\eqref{deltaprimeprime}}{\leq} \frac{2 e^{\|\bxi\|_2^2}(2M\|\bxi\|_2+2+\delta)}{n\delta} + \frac{16 e^{\|\bxi\|_2^2}(2M\|\bxi\|_2+2+\delta)^2}{n\delta^2}
\end{align*}
Adding the above bounds yields 
\begin{align}
\prob\Big\{&\big\|\bw(\bxi)-\bow(\bxi)\big\|_2\ge \delta\Big\}\leq \frac{e^{\|\bxi\|_2^2}}{n\delta^2}\left[(10M^2d+49M^2\|\bxi\|_2^2+17)(2M\sqrt{d+\|\bxi\|^2_2}+2+\delta)^2\right]+\nonumber\\
&\qquad+\frac{e^{\|\bxi\|_2^2}}{n\delta}[(2M^2+3d+12\|\bxi\|_2^2+2)(2M\|\bxi\|_2+2+\delta)]\label{final}
\end{align}
and the lemma follows. \qed
%\end{proof}
%The proof of this lemma can be found in Appendix~\ref{app:concentration}. 

\section{Gradient Coefficients and Approximate Normality}\label{app:gradnorm}
\subsection{Proof of Lemma~\ref{lemma:coeffboundssimple}}\label{proofofcoeffbounds}
%\begin{proof}
Observe that for $f(x)=\tanh(\beta x)$ we have:
%\
\begin{align}
\label{eq:BoundsFprime}
\beta e^{-2\beta |x|}<f'(x) = \frac{4\beta e^{-2\beta |x|}}{(1+
  e^{-2\beta |x|})^2} \le 4\beta e^{-2\beta|x|}. 
\end{align}
 Observe also that $X\equiv \<\bw^\ell,\bX\>$ is a 1-dimensional zero-mean
 Gaussian r.v.~with variance $ \|\bw_i\|^2_2=1$. 
Using the upper bound in Eq.~(\ref{eq:BoundsFprime}), we get
\begin{align*}
\E_{0}\big\{f'\big(\<\bw^\ell,\bxi+\bX\>\big)\big\} & =
\E\big\{f'(z_\ell+X)\big\}\le 4\beta\E\big\{e^{-2\beta |z_\ell+X|}\big\}
\stackrel{(a)}{\le} 8\beta \E\big\{e^{-2\beta (|z_\ell|+X)}\ind(X\ge
-|z_\ell|)\big\}\\
& \le
8\beta\ e^{-2\beta|z_\ell|} \E\{e^{-2\beta X}\} = 8\beta\,
e^{-2\beta|z_\ell| + 2\beta^2}\, ,
\end{align*}
where in $(a)$ we used $\E\{e^{2\beta (|z_\ell|+X)}\ind(X\le -|z_i|)\}\le \E\{e^{-2\beta (|z_\ell|+X)}\ind(X\ge -|z_\ell|)\}$.
This proves the upper bound. 
Similarly, the lower bound on $f'$ yields:
\begin{align*}
\E\big\{f'(z_\ell+X)\big\}&\ge \beta\E\big\{e^{-2\beta |z_\ell+X|}\big\}
 \ge \beta \E\big\{e^{-2\beta (|z_\ell|+X)}\ind(X\ge -|z_\ell|)\big\}\\
& \stackrel{(b)}{=} \beta e^{-2\beta|z_\ell|+2\beta^2} \E_{-2\beta}
\ind(X\ge -|z_\ell|) = \beta e^{-2\beta|z_\ell|+2\beta^2} \prob(X\ge -|z_\ell|+2\beta)\, ,
\end{align*}
where the equality $(b)$ follows from the Gaussian integration formula
\eqref{expectation} (with $\E_{-2\beta}$ denoting expectation with
respect to $X\sim\normal(-2\beta,1)$). \qed
%\end{proof}

\subsection{Proof of Lemma~\ref{anglelemma}}\label{proofofanglelemma}
%\begin{proof} 
The first statement of the Lemma follows  by observing that $\<\bw^\ell,\bXi\>/\xi_0$ is a standard Gaussian.
To prove the second statement, we establish an auxiliary result:
\begin{lemma}\label{parallelogramlemma}Let $(Z_1,Z_2)\in\reals^2$ be a zero-mean Gaussian random variable with covariance $\Sigma=\sigma^2\left[\begin{smallmatrix}1& c\\c& 1\end{smallmatrix}\right]$, for some $|c|<1$. Then 
$\prob (|Z_1|<\epsilon_1, |Z_2|<\epsilon_2) \leq \frac{2\epsilon_1\epsilon_2}{\pi\sigma^2\sqrt{1-c^2}}.$
\end{lemma}
\begin{proof}
Observe that $Z=\Sigma^{1/2}W$, where $W$ is a standard Gaussian. Hence, 
$\prob (|Z_1|<\epsilon, |Z_2|<\epsilon)=\prob (W\in \mathcal R)\leq \frac{|\mathcal R|}{2\pi},$ where $\mathcal R$ the parallelogram defined by the endpoints $\Sigma^{-1/2}\left[\begin{smallmatrix} \pm \epsilon_1\\\pm \epsilon_2\end{smallmatrix}\right]$. The area  $|\mathcal R|$ is given by the determinant of the matrix comprising the two vectors defining $\mathcal{R}$; as such, it is 
$|R|= \det\left(\Sigma^{-1/2}\left[\begin{smallmatrix}2\epsilon_1 &0\\0 & 2\epsilon_2 \end{smallmatrix}\right]\right)= (4\epsilon_1\epsilon_2)(\det(\Sigma))^{-1/2}=\frac{4\epsilon_1\epsilon_2}{\sigma^2\sqrt{1-c^2}}.$
\end{proof}
%\begin{proof}
 Observe that $\<\bw^\ell,\bXi\>$ and $\<\bw^{\ell'},\bXi\>$ are jointly Gaussian with zero mean and covariance $\Sigma=\xi_0^2\left[\begin{smallmatrix}1& c\\c& 1\end{smallmatrix}\right]$, with $c=\<\bw^\ell,\bw^{\ell'}\>$. Hence, the second statement follows by Lemma~\ref{parallelogramlemma}, as the latter implies:
$$\prob(|\<\bw^\ell,\bXi\>|<\Delta_1,|\<\bw^{\ell'},\bXi\>|<\Delta_2)
\leq \frac{2\Delta_1\Delta_2}{\pi \xi_0^2 \sqrt{1-\<\bw^\ell,\bw^{\ell'}\>^2}},\quad\text{for all }\ell\neq \ell'\text{ in } [k].$$
Recall that $\bM$ is the $d \times k$ matrix whose columns comprise all vectors $\bw^{\ell}$, $\ell\in [k]$, and let $\bM_{\ell\ell'}$ be the $d\times 2$ matrix comprising only vectors $\bw^\ell$ and $\bw^{\ell'}$. Notice that $\bM_{\ell\ell'}^T\bM_{\ell\ell'}$ is a principal submatrix of $\bM^T\bM$. Hence, by the Cauchy interlacing theorem, 
$$\sigma_{\min}(\bM_{\ell\ell'}^T\bM_{\ell\ell'} ) \geq  \sigma_{\min}(\bM^T\bM) =(\sigma_{\min}(\bM))^2\geq \kappa^{2} .$$
 %\leq  \sigma_{\max}(\bM_{ij}^T\bM_{ij})  \leq \sigma_{\max}(\bM^T\bM)\leq k.  $$
On the other hand, $1-\<\bw^\ell,\bw^{\ell'}\>^2= \det(\bM_{\ell\ell'}^T\bM_{\ell\ell'})\geq \sigma_{\min}(\bM_{\ell\ell'}^T\bM_{\ell\ell'})$; the last inequality follows from the fact that the trace of $\bM_{\ell\ell'}^T\bM_{\ell\ell'}$ is 2 and, thus, at least one of its eigenvalues is at least 1.  Hence, the second statement of the lemma follows.\qed
 %Observe first that
%$$\prob(|\< w_i,\bXi\>|<\Delta)=\prob(|\cos(\angle(\bw_i,\bXi))|<\frac{\Delta}r)=\frac{1}{\pi}\left(\frac{\pi}2-\arccos\frac{\Delta}{r}\right)$$
%which yields the first statement.
%For the second statement, observe that $\bXi\sim \frac{\bzeta r}{\|\bzeta\|_2}$, where $\bzeta\in \reals^d$ a standard Gaussian. \begin{align*}
%\prob(|\<\bw_i,\bXi\>|<\Delta,|\<\bw_j,\bxi\>|<\Delta)
%& = \prob(|\<\bw_i,\bzeta\>|<\frac{\Delta\|\bzeta\|_2}{r},|\<\bw_j,\bzeta\>|<\frac{\Delta\|\bzeta\|_2}{r})\\
%& \leq \prob(|\<\bw_i,\bzeta\>|<\frac{\Delta r_0}{r},|\<\bw_j,\bzeta\>|<\frac{\Delta r_0}{r})+ \prob(\|\bzeta\|_2>r_0)
%\end{align*} 
%\end{proof}

\section{Proofs of Lemmas Bounding the Size of Each Partition}
\subsection{Proof of Lemma~\ref{clusterlemma}}\label{proofofclusterlemma}
%\begin{proof} 
Note that 
%
%\begin{align*}
%
$\big\|\btw(\bxi_{j}) -\bw^{\ell}\big\|_2  =
\Big\|\frac{\bw(\bxi_j)}{\|\bw(\bxi_j)\|_2} -\bw^{\ell}\Big\|_2
 \le  \Big\|\frac{\bw(\bxi_j)}{\|\bw(\bxi_j)\|_2} - \frac{\bow(\bxi_j)}{\|\bow(\bxi_j)\|_2} \Big\|_2+
\Big\|\frac{\bow(\bxi_j)}{\|\bow(\bxi_j)\|_2} -\bw^{\ell}\Big\|_2 \stackrel{(a)}{\le} 2\delta+ \frac{2\delta}{1-\delta} \stackrel{(b)}{\leq }6\delta.$
%
%\end{align*}:!
Here, the  first term in bound $(a)$ follows from 
\begin{align*} \Big\|\frac{\bw(\bxi_j)}{\|\bw(\bxi_j)\|_2} - \frac{\bow(\bxi_j)}{\|\bow(\bxi_j)\|_2} \Big\|_2 &\leq \frac{\|\bw(\bxi_j)-\bow(\bxi_j)\|_2}{\|\bw(\bxi_j)\|_2} + \|\bow(\bxi_j)\|_2\big|\frac{1}{\|\bw(\bxi_j)|_2} -\frac{1}{\|\bow(\bxi_j)\|_2 }\big| \\&= \frac{\|\bw(\bxi_j)-\bow(\bxi_j)\|_2}{\|\bw(\bxi_j)\|_2} - \frac{| \|\bw(\bxi_j)\|_2-\|\bow(\bxi_j)\|_2 |}{\|\bw(\bxi_j)\|_2}\leq 2\delta\end{align*}
as for $j\in\cC_{\ell}$, we have
$\|\bw(\bxi_j)\|_2\ge w_0$ and
 since   $\cC_{\ell}\subseteq\cG$, we have $| \|\bw(\bxi_j)\|_2 - \|\bow(\bxi_j)\|_2 |\leq\|\bw(\bxi_j) - \bow(\bxi_j)\|_2\leq \delta w_0.$
The second term in bound $(a)$ follows from \eqref{allbutone}. Indeed,  since $\bxi_j\in \cR_\ell$,
%Note that \eqref{allbutone} implies that:
%
%\begin{align}
%\|\bow(\bxi)\|_2\ge \frac{w_0}{2}\, \Rightarrow
%\Big\|\frac{\bow(\bxi)}{\|\bow(\bxi)\|_2}-\bw^{\ell}\Big\|_2\le
%4\delta\, . \label{eq:InGoodRegion}
%\end{align}
%since 
$\Big\|\frac{\bow(\bxi_j)}{\|\bow(\bxi_j)\|_2}-\bw^{\ell}\Big\|_2 \leq \frac{\|\bow(\bxi_j) -a_\ell \bw^{\ell}\|_2 }{\|\bow(\bxi_j)\|_2}-\frac{|a_\ell - \|\bow(\bxi_j)\||}{ \|\bow(\bxi_j)\|_2}\stackrel{\eqref{allbutone}}{\leq} \frac{2 \delta w_0}{ \|\bow(\bxi_j)\|_2}.  $
On the other hand, since $|\|\bw(\bxi_j)\|_2 - \|\bow(\bxi_j)\|_2|\leq \delta w_0$ and $\|\bw(\bxi_j)\|_2\ge w_0$, we have that $\|\bow(\bxi_j)\|_2 \geq (1-\delta) w_0$, so the second bound of $(a)$ holds. Finally, $(b)$ follows from 
 $\delta \in (0,0.5]$. \qed
%\end{proof}

\subsection{Proof of Lemma~\ref{spuriouslemma}}\label{proofofspurious}
%\begin{proof}
We have 
\begin{align}
\cC_0 &\subseteq\cG^c\cup \cB_0\cup\cB_*\, ,\label{eq:CoInclusion}\\
\cB_0 & \equiv \Big\{j\in \cG:\, \bxi_j\in \cR_0, \;
\|\bow(\bxi_j)\|_2\ge (1-\delta)w_0\Big\}\, ,\\
\cB_* & \equiv \Big\{j\in \cG:\, \bxi_j\in \cR_*, \;
\|\bow(\bxi_j)\|_2\ge (1-\delta)w_0\Big\}\, ,
\end{align}
since, for  $j\in\cG$, the event $j\not\in\cup_{\ell=1}^k\cC_{\ell}$ implies
$\bxi_k\in\cR_0\cup\cR_*$.
Further $\|\bw(\bxi_j)\|_2\ge w_0$ implies $\|\bow(\bxi_j)\|_2\ge
(1-\delta )w_0$ because --by definition of $\cG$--
$\|\bow(\bxi_j)-\bw(\bxi_j)\|_2\le \delta w_0$.

From Eq.~(\ref{eq:CoInclusion}), $|\cC_0|\le |\cG^c|+|\cB_0|+\cB_*|$. Note that $\cB_0=\emptyset$ by construction, due to Eq.~(\ref{eq:InR0}).
On the other hand, %
%\begin{align}
%
$
\cB_*\subseteq \cB_*' \equiv \Big\{j\in [m_0]:\, \bxi_j\in
\cR_*\Big\}\, .
$	
%
%\end{align}
%
Thus $|\cB_*'|$ is a binomial random variable with $m_0$ trials and
success probability
%
%\begin{align}
%
$
\prob\Big(\bxi_1\in\cR_*\Big)\stackrel{(a)}{\le} \frac{2 k^2}{\pi\kappa}
\left(\frac{\Delta}{\xi_0}\right)^2\stackrel{(b)}{=}\gamma\rho ,
$
%
%\end{align}
%
where $(a)$ is implied by Lemma~\ref{anglelemma} and $(b)$ is by construction of $\xi_0$ and $\gamma$---c.f.~\eqref{xi0} and \eqref{gamma}. 
Hence, for any $\epsilon\leq 2e-1$, we get the Chernoff bound
%
%\begin{align}
%
$
\prob\left( |\cB_*|\ge m_0
\gamma\rho (1+\epsilon)\right) \le e^{-\epsilon^2  m_0 \gamma\rho/4}.
$
%
%\end{align}
%
Hence, we have that
$\prob\Big(|\cC_0|>2m_0\gamma\rho)\leq \prob\big(|\cG^c|\ge \frac{m_0\gamma\rho}{2}\big)+\prob\left( |\cB_*|\ge m_0
\gamma\rho\frac{3}{2}\right) \leq \prob\big(|\cG^c|\ge \frac{m_0\gamma\rho}{2}\big) +  e^{- m_0 \gamma\rho/16}.
$

To obtain the two statements in the lemma, it therefore remains to bound size of $\cG^c$. By Markov's
inequality 
\begin{align*}
\prob\big(|\cG^c|\ge m_0\eps\big)&\le \frac{\E\{|\cG^c|\}}{m_0\eps}  =
\frac{1}{\eps} \prob\Big(\|\bw(\bxi_1)-\bow(\bxi_1)\|_2>
\delta{w_0}\Big)
\end{align*}
Thus, under the Value Oracle model,  \eqref{voebound} in Lemma~\ref{assympboundsvoestimate} directly gives:
\begin{align}\label{GCboundvo}
\prob\big(|\cG^c|\ge m_0\eps\big)&\le\frac{c_1}{\eps}
\exp\Big(-\min \big\{ \frac{c_2n_0\delta^2w_0^2 }{d M^2} , (c_3\frac{\sqrt{n_0}\delta w_0}{M} - c_4 \sqrt{d})^2\big\}\Big)
\end{align}
 the first statement immediately follows. 

To prove the second statement, assume that the Gaussian Covariates model, and \eqref{eq:Slope}, is used to produce $\bw(\bxi)$ instead. Note that
\begin{align*}
\prob\big(|\cG^c|\ge m_0\eps\big)& \le \frac{1}{\eps} \left(   \prob\Big(\|\bw(\bxi_1)-\bow(\bxi_1)\|_2>
\delta{w_0}\mid \|\bxi_1\|_2^2 \le \alpha \xi_0^2\Big)  +\prob(\|\bxi_1\|_2^2> \alpha \xi_0^2)  \right)\\
& \stackrel{(a)}{\le} \frac{1}{\eps} \left[ \frac{C}{n\delta^2w_0^2} M^4 \big(e^{\alpha \xi_0^2}
(d+\alpha\xi_0^2)^2\big) +  \prob(\|\bxi_1\|_2^2> \alpha \xi_0^2) \right] \\ 
&\stackrel{(b)}{\le}  \frac{CM^4 d^2}{n\delta^2w_0^2\eps}
\big\{e^{(1+2d^{-1}) \alpha\xi_0^2}\big\} + \frac{1}{\eps}  \prob(\|\bxi_1\|_2^2> \alpha \xi_0^2)
%&\le   \frac{CM^4 d^2}{n\delta^2w_0^2\eps}
%\Big(1-(2+4d^{-1})\xi_0^2\Big)^{-d/2}\, ,
%
\end{align*}
where $C$ is a numerical constant, $(a)$ follows from Corollary \ref{assympboundscor} and, in
$(b)$, we used $(1+x)\le e^x$. On the other hand, the square of the norm of a standard Gaussian follows a chi-squared distribution, so by \cite{dasgupta2003elementary}, for $\alpha>d$ we get that:
%\begin{align} 
$ \prob(\|\bxi_1\|_2^2> \alpha \xi_0^2)\leq \left(\frac{\alpha}{d} e^{1-\frac{\alpha}{d}} \right)^{\frac{d}{2}} = e^{-\frac{d}{2} (\frac{\alpha}{d} -1 - \log \frac{\alpha}{d}) }.$
Note that  $\log x/x$ is increasing in $(0,e]$ and decreasing in $[e,+\infty)$. Hence, for all $x>0$,  $\log x\leq \frac{x}{e}$ and, thus,
%\begin{align} 
 $\prob(\|\bxi_1\|_2^2> \alpha \xi_0^2)\leq e^{-\frac{d}{2}(\frac{e-1}{e}\frac{\alpha}{d} -1)} \leq e^{-\frac{e-1}{4e}\alpha} \leq e^{\frac{\alpha}{7}}$
%\end{align}
for all $\alpha$ such that
$\frac{\alpha}{d}\geq 4 >\frac{2e}{e-1}$.
Hence, under this condition on $\alpha$, we have that
%\begin{align}
$\prob\big(|\cG^c|\ge m_0\eps\big)%&\le \frac{CM^4 d^2}{n\delta^2w_0^2\eps}
%e^{(1+2d^{-1}) \alpha\xi_0^2}
% + \frac{\int_\alpha^\infty e^{-\frac{x}{4}}dx}{ 2^{\frac{d}{2}}\Gamma(\frac{d}{2})\varepsilon}\\
%&
\le \frac{CM^4 d^2}{n\delta^2w_0^2\eps}
e^{(1+2d^{-1}) \alpha\xi_0^2}
 +  \frac{1} { \varepsilon}e^{\frac{\alpha}{7}}$
%\end{align}
This means that by setting $\alpha = \frac{1}{\frac{1}{7}+   (1+2d^{-1}) \xi_0^2     }\log{   \frac{n\delta^2w_0^2}{M^4 d^2}} > 4d$
we get that, for $C$ an absolute constant:
\begin{align}
\prob\big(|\cG^c|\ge m_0\eps\big)\le \frac{C}{\eps} \left(\frac{M^4 d^2}{n\delta^2w_0^2} \right)^{\frac{1}{1 +7(1+2d^{-1})\xi_0^2}}
\, . \label{GCbound}
\end{align}
and the second statement follows.\qed
%\end{proof}

\subsection{Proof of Lemma~\ref{goodclusterlemma}}\label{proofofgoodclusterlemma}
%\begin{proof}
We bound the size of $\cC_{\ell}$ as
follows:
%
%\begin{align}
%
$$|\cC_{\ell}|\ge |\cC'_{\ell}|-|\cG^c|\, ,$$ where
$\cC_{\ell}'  \equiv  \Big\{j\in [m_0]:\, \bxi_j\in \cR_\ell', \;
\|\bow(\bxi_j)\|_2\ge w_0\Big\}\,$ and
%
%\end{align}
%\begin{align}
$\cR_{\ell}' \equiv \big\{\bxi\in\reals^d:\, \,
|\<\bw^{\ell},\bxi\>|<\frac{\Delta}{2},\;
\min_{i\in [k]\setminus \ell}|\<\bw^{i},\bxi\>|\ge\Delta\big\}\,.
$
%\end{align}
%
We only need to lower bound $|\cC'_{\ell}|$, as $|\cG^c|$ can be upper-bounded by \eqref{GCboundvo} or \eqref{GCbound}, under the Value Oracle and Gaussian Covariates model, respectively. 
Observe first that, for any $\bxi_j\in \cR_{\ell}'$, as in  \eqref{allbutone}, we have that 
$\|\bow(\bxi)\|\geq |a_{\ell}|- \delta w_0$
where $|a_\ell|\geq  c_1u_{0}e^{-2\beta\frac{\Delta}{2}}. $ 
 On the other hand, for $\Delta$ satisfying \eqref{Delta}
we get that
$ c_1u_{0}e^{-2\beta\frac{\Delta}{2}} = (1+\delta)w_0. $
Hence,
$ \bxi_j\in \cR_{\ell}'\Rightarrow \|\bow(\bxi_j)\|_2\ge w_0. $
Moreover, since $\Delta<\bxi_0$ by \eqref{xi0}, Lemma~\ref{anglelemma} implies that $|\cC'_{\ell}|$
a binomial random variable with success probability 
$
\prob\Big(\bxi_1\in\cR_\ell'\Big) \geq \sqrt{\frac{2}{e\pi}}\frac{\Delta}{2\xi_0} - k  \frac{2 }{\kappa\pi}
\left(\frac{\Delta}{\xi_0}\right)^2\,=\gamma.
$
Hence, for any $\epsilon\in(0,1]$, we have
$ 
\prob\left( |\cC'_\ell|\leq m_0 \gamma
 (1-\epsilon)\right) \le e^{-\frac{\epsilon^2}{2} m_0 \gamma }.
$\qed
%\end{proof}

\section{Proof of Theorem~\ref{dimredtheorem}}
Let $\mathcal{M}=\linspan(\bM)$, and $\hat{\mathcal{M}}$ the estimate of $\mathcal{M}$, suppose that the largest principal angle between the two spaces satisfies $$d_p(\mathcal{M},\hat{\mathcal{M}})\leq\theta\leq \frac{\pi}{2}.$$ Then, there exist orthonormal bases  $\{e_\ell\}_{\ell=1}^{k}$ , $\{\hat{e}_\ell\}_{\ell=1}^{k}$  of  $\mathcal{M}$, $\hat{\mathcal{M}}$, respectively so that
\begin{align} \< e_\ell,\hat{e}_{\ell}\> \leq \cos \theta, \text{  for all }\ell\in\{1,\ldots,k\}.\label{basisangles}\end{align}
Note that \eqref{basisangles} immediately implies that
\begin{align}\label{basisnorms}
\|e_\ell-\hat{e}_{\ell}\|_2 \leq 2\sin\frac{\theta}{2}, \text{  for all }\ell\in \{1,\ldots,k\}.
\end{align}
Denote by  $P,\hat{P}\in\reals^{d\times k }$ the matrices comprising the above orthonormal bases as columns. The projections to $\mathcal{M}$,$\hat{\mathcal{M}}$ can then be written as
\begin{align*}
P_\mathcal{M}(\bx) = PP^{\top} \bx,  & &\text{and} && P_{\hat{\mathcal{M}}}(\bx) = \hat{P}\hat{P}^{\top} \bx.
\end{align*}
The following lemma holds:
\begin{lemma}\label{projectionlemma}For all $\bw\in \mathcal{M}$ with $\|\bw\|_2=1$ and all $\bx \in \reals^d$, 
$$\< \bw,\bx \>- \< P_{\hat{\mathcal{M}}}(\bw),P_{\hat{\mathcal{M}}}(\bx)\> = \< \bw,\bx \>- \< \hat{P}^\top \bw,\hat{P}^\top\bx\> = \< \bd, \bx\>$$ where $\|\bd\|_2\leq 4k\sin\frac{\theta}{2}. $
In particular, for all unit-norm $\bw,\bw'\in\mathcal{M}$,$$ | \< \bw,\bw' \>- \< P_{\hat{\mathcal{M}}}(\bw),P_{\hat{\mathcal{M}}}(\bw')\>| =|\< \bw,\bw'\>- \< \hat{P}^\top \bw,\hat{P}^\top\bw'\>   | \leq  4k\sin\frac{\theta}{2} .$$
\end{lemma}
\begin{proof}
Since $\bw\in \mathcal{M}$,
$$\< \bw, \bx \> = \< PP^\top\bw,PP^\top\bw\> =\< P^\top\bw,P^\top\bw\> $$
as $P^\top P = I$. On the other hand, we have that:
\begin{align*}
 \< P^\top\bw,P^\top\bx\> &= \< \hat{P}^\top \bw, \hat{P}^\top\bx \> - \< \hat{P}^\top \bw,(\hat{P}^\top-P^\top)\bx \>  - \< (\hat{P}^\top -{P}^\top)\bw, P^\top\bx  \>\\
& = \< \hat{P}^\top\bw, \hat{P}^\top\bx \> + \< \bd,\bx\>
\end{align*}
where $\bd\in \reals^{1\times d}$  is a vector with $\|\bd\|_2\leq ( \|\hat{P}\|_2+\|P\|_2)\|\hat{P} -{P} \|_2 \|\bw\|_2 =  2\sqrt{k} \cdot 2\sqrt{k}\sin\frac{\theta}{2}\cdot 1. $ The lemma follows again as $\hat{P}^\top\hat{P}=I$.
\end{proof}
\begin{corollary}\label{anglecor}
For any $\bw\in\mathcal{M}$ s.t.~$\|\bw\|_2=1$, $\|\bw-P_{\hat{\mathcal{M}}}(\bw)\|_2\leq 2\sqrt{k\sin\frac{\theta}{2}}.  $ 
\end{corollary}
\begin{proof}
From Lemma \ref{projectionlemma} we have that $\|\bw\|_2^2- \|P_{\hat{\mathcal{M}}}(\bw)\|_2^2=|\|\bw\|_2^2- \|P_{\hat{\mathcal{M}}}(\bw)\|_2^2|\leq 4k\sin\frac{\theta}{2},$ where the first equality holds because projections are contractions. Hence
\begin{align*}\|\bw-P_{\hat{\mathcal{M}}}(\bw)\|_2^2 = \|\bw\|^2_2+\|P_{\hat{\mathcal{M}}}(\bw)\|_2^2 - 2\<\bw,P_{\hat{\mathcal{M}}}(\bw) \>= \|\bw\|^2_2-\|P_{\hat{\mathcal{M}}}(\bw)\|_2^2 \leq  4k\sin\frac{\theta}{2}.
\end{align*}
\end{proof}
 For every $\bx\in\reals^{d}$, denote by $\bhx$ the projection of $\bx$ to $\hat{\mathcal{M}},$ i.e.,
$$ \bhx = P_{\hat{\mathcal{M}}}(\bx) = \hat{P}\hat{P}^{\top} \bx.$$

 Then, the following holds:
\begin{lemma}[Concentration Bound under Dimensionality Reduction]\label{subspace}
There exists a numerical constant $C$ such that,
when $\bw(\bxi)$ is computed through \eqref{projectedestimate}, for any fixed $\bxi\in\reals^d$:
\begin{align}
\prob\Big\{\big\|\bw(\bxi)-\bhw(\bxi)\big\|_2\ge \delta\Big\}&\leq
\frac{Ce^{\|\bhxi\|_2^2}}{n\delta^2}\, M^4(k+\|\bhxi\|^2)^2\, .
\label{candidatebounddim}
\end{align}
where
\begin{align}\label{populationdim}
\bhw(\bxi) =  \hat{P}\hat{P}^\top\E_\bhxi \big\{\nabla r(\bX)\big\}  =\textstyle\sum_{\ell=1}^k
u_\ell\bhw^\ell\E_{\bhxi}\big\{f'\big(\<\bw^\ell,\bX\>\big)\big\}\, ,
\end{align}
for $\E_{\bxi}$ denoting the expectation w.r.t.~$ \bX\sim\normal(\bxi,\id_{d\times d})$.
\end{lemma}
The proof follows, mutatis mutandis, the same steps as the proof of Lemma~\ref{assympboundscor}, so we ommit it for brevity. The next lemma is the equivalent of Lemma~\ref{anglelemma}, for the case where $\bXi$ is first projected to $\hat{\mathcal{M}}$.
\begin{lemma}\label{anglelemma2} Assume that $\bXi\in \reals^d$ is sampled from $\normal(0,\xi_0^2\,\id_{d})$. Then, for any $0<\Delta<\xi_0$,
$\prob(|\< \bw^\ell,\hat{\bXi}\>|<\Delta)\geq \sqrt{\frac{2}{e\pi}}\frac{\Delta}{\xi_0},$ for all $\ell\in[k] $
and for any $\Delta_1,\Delta_2>0$, 
$\prob(|\<\bw^\ell,\hat{\bXi}\>|<\Delta_1,|\<\bw^{\ell'},\hat{\bXi}\>|<\Delta_2)
\leq \frac{2\Delta_1\Delta_2}{\pi \xi_0^2 \sqrt{\kappa^{2}-\big(4k\sin\frac{\theta}{2}\big)^2}},$ for all $\ell\neq \ell'$ in $[k]$,
where $\kappa\leq \sigma_{\min}(\bM)$.
\end{lemma}
\begin{proof}
The first statement  can be shown as in Lemma~\ref{anglelemma} using the fact that $\<\bw^\ell,\hat{\bXi}\>=\<\bhw^\ell,\bXi\>$, and that $\|\hat{\bw}^\ell\|_2\leq 1$, as $P_{\hat{\mathcal{M}}}$ is a contraction.  Similarly, to prove the second statement, we have that, for all $i,j\in [k]$,
\begin{align*}
\prob(|\<\bw^\ell,\hat{\bXi}\>|<\Delta_1,|\<\bw^{\ell'},\hat{\bXi}\>|<\Delta_2)  &= \prob(|\<\bhw^{\ell},\bXi\>|<\Delta_1,|\<\bhw^{\ell'},\bXi\>|<\Delta_2)\leq  \frac{2\Delta_1\Delta_2}{\pi \xi_0^2 \sqrt{1-\<\bhw^\ell,\bhw^{\ell'}\>^2}}
\end{align*}
 where the last inequality follows from Lemma~\ref{parallelogramlemma}. On the other hand, by Lemma~\ref{projectionlemma}:
$$1-\<\bhw^\ell,\bhw^{\ell'}\>^2\geq 1- \<\bw^\ell,\bw^{\ell'}\>^2-  (2k \sin\frac{\theta}{2})^2, $$ 
and the lemma follows as  $1- \<\bw^\ell,\bw^\ell\>^2\geq \kappa^{2}$.
\end{proof}

We can now describe how the 
candidate indices
 $\cC \subset [m_0] $ produced by Algorithm~\ref{algo:Candidates}  can be partitioned as
$\cC = \cC_0 \cup\cC_1\cup \cdots\cup\cC_k\, ,$ 
s.t.~for any $i\in \cC_\ell$, candidate $\btw(\bxi^{(i)})$ is close to $\bw^\ell$, while $\cC_0$ is a small set of spurious candidates.

Given $\delta\in(0,0.5],$ and $\rho\in(0,1)$, we take $\Delta$, $w_0$ as in \eqref{Delta} and \eqref{w0}, respectively. We take also $\xi_0$ and $\gamma$ as in \eqref{xi0} and \eqref{gamma}, with the only difference that $\kappa$ is replaced by
\begin{align}
\label{kappaprime} 
\kappa'=\sqrt{\kappa^2 - \big(4k\sin\frac{\theta}{2}\big)^2}
\end{align}

Then, for sets $\cR_0$, $\cR_{\ell},\ell\in[k]$, and $\cR_{*}$ defined as in \eqref{Rsets}, we can again show that, instead of \eqref{eq:InR0} and \eqref{allbutone}, we have that
$$ \|\bhw(\bxi)\|_2\le k\, c_2\, e^{-2\beta\Delta} \stackrel{\eqref{w0}}{=} \delta w_0,\quad \text{for all } \bxi\in \cR_0,$$
while
$$ \|\bhw(\bxi)-a_{\ell}\bhw^{\ell}\|_2\le k\, c_2\, e^{-2\beta\Delta} = \delta w_0,\quad \text{for all } \bxi\in \cR_{\ell}. $$ The following lemmas can thus be proved using the same steps as in Section~\ref{sec:partitioning}, using the bounds in Lemma~\ref{subspace} rather than the bounds in Lemma~\ref{assympboundscor}.
\begin{lemma}\label{clusterlemma2}
 For each $\ell\in [k]$ and each $j\in\cC_{\ell}$,
  $\big\|\btw(\bxi_{j}) -\bhw^{\ell}\big\|_2\le 6\delta$.
\end{lemma}
\begin{lemma}\label{spuriouslemma2}  The event $|\cC_0|\le  2 \gamma \rho m_0 $ occurs   with  probability at least $1 - \Big(
\frac{c_1}{   \gamma \rho  } \left(\frac{M^4 k^2}{n\delta^2w_0^2} \right)^{\frac{1}{1 +4(1+2k^{-1})\xi_0^2}} 
%\frac{C_1M^4 d^2}{n\delta^2w_0^2\gamma\rho}
%e^{(2d+4)\xi_0^2}
 +  e^{-c_2 m_0 \gamma\rho}\Big)$, for $n>  \frac{M^4 k^2}{\delta^2w_0^2} e^{  4k  \left(\frac{1}{7}+   (1+2k^{-1}) \xi_0^2\right)  }$ and $c_1, c_2$ absolute constants.\end{lemma}

\begin{lemma} \label{goodclusterlemma2}
For $\ell\in [k]$, the event that  $|\cC_{\ell}|\ge m_0\gamma/2 $,  occurs  with probability at least  $1-\left(
\frac{c_1}{   \gamma \rho  } \left(\frac{M^4 k^2}{n\delta^2w_0^2} \right)^{\frac{1}{1 +7(1+2k^{-1})\xi_0^2}}
%\frac{C_3M^4 d^2}{n\delta^2w_0^2\gamma\rho}
%e^{(2d+4)\xi_0^2}
+ e^{-c_3m_0\gamma}\right)$, where $c_1,c_2$ are absolute constants, for  $n>  \frac{M^4 k^2}{\delta^2w_0^2} e^{   4k\left(\frac{1}{7}+   (1+2k^{-1}) \xi_0^2\right)  }$ \end{lemma}

Let $\theta^*$ be such that the following inequalities hold
\begin{align*}
4 k \sin \frac{\theta^*}{2}&\leq \sqrt{\frac{3}{4}}\kappa, & 
2\sqrt{k \sin \frac{\theta^*}{2}}&\leq \delta, &
\theta^* &\leq u_0\kappa. 
\end{align*}
Note that these are satisfied for
$$\theta^* = \min \left\{ 2\arcsin \frac{\sqrt{3}\kappa}{8k },2\arcsin \frac{\delta^2}{4k}, u_0\kappa \right\}.$$
Note that, if an estimate of $\mathcal{M}$ s.t.~the largest principal angle between  $\mathcal{M}$ and $\hat{\mathcal{M}}$ is $\theta^*$, then by Corollary~\ref{anglecor}:
$$\|\bw_\ell -\bhw_\ell\|\leq \delta, \quad \text{for all }\ell\in[k],$$
and
$$
\kappa'=\sqrt{\kappa^2 - \big(4k\sin\frac{\theta^*}{2}\big)^2} \geq \frac{\kappa}{2}. $$

Putting everything together, by Theorem~\ref{sunthm}, if $$n_1>c\frac{d}{\left(\theta^*\right)^2}$$
 samples are used to estimate the subspace, 
$$n_2> \frac{M^4 k^2}{\delta^2w_0^2}\max\left(C'' \left(\frac{k}{\gamma\rho\delta}\right)^ {1 +7(1+2k^{-1})\xi_0^2}, e^{  4k \left(\frac{1}{7}+   (1+2k^{-1}) \xi_0^2\right)  } \right) $$
are used in the gradient estimation, and
$$ m_0>C\frac{1}{\gamma\rho}\log\frac{k}{\delta} $$ 
samples are used in the candidate generation, then with probability at least $1-\delta$ 
 the set of candidate indices
 $\cC \subseteq [m_0]$,  can be partitioned as
%
%\begin{align}
%
$$\cC = \cC_0 \cup\cC_1\cup \cdots\cup\cC_k\, ,$$ %\label{eq:PartitionC0Cl}
%
%\end{align}
%
where 
for any
$\ell\in\{1,2,\dots, k\}$, if  $i\in\cC_\ell$  then
%
%\begin{align}
%
%i\in \cC_\ell \;\;\Rightarrow \;\;
$$ \big\|\btw(\bxi^{i}) -\bw^{\ell}\big\|_2\le 7 \delta\, ,$$
%
%\end{align}
while 
$\cC_0$ is a set of `bad' candidates, such that
 $|\cC_0|\le 2\rho \gamma m_0,$ and 
$|\cC_\ell|\geq \gamma m_0/2$ for all $\ell \in[k].$
\qed

\end{document}